\documentclass{article}

\usepackage[utf8]{inputenc} 
\usepackage[T1]{fontenc}    
\usepackage{hyperref}       
\usepackage{url}            
\usepackage{booktabs}       
\usepackage{amsfonts}       
\usepackage{microtype}      
\usepackage[square,numbers,sort&compress]{natbib}

\usepackage{graphicx}
\usepackage{mathtools}
\usepackage{amssymb}
\usepackage{amsthm}
\usepackage{paralist}
\usepackage{subcaption}
\usepackage{titlesec}
\usepackage[dvipsnames]{xcolor}
\usepackage[nameinlink,capitalise]{cleveref}
\usepackage[acronym,smallcaps,nowarn,section,nogroupskip,nonumberlist]{glossaries}
\usepackage{todonotes}
\usepackage{thmtools}
\usepackage{thm-restate}
\usepackage{tikz}

\definecolor{red}{HTML}{E41A1C}
\definecolor{orange}{HTML}{FF7F00}
\definecolor{yellow}{HTML}{FFC020}
\definecolor{green}{HTML}{4DAF4A}
\definecolor{blue}{HTML}{377EB8}
\definecolor{purple}{HTML}{984EA3}
\graphicspath{{images/}} 

\hypersetup{%
  colorlinks,
  linkcolor={violet!70!black},
  citecolor={YellowOrange!70!black},
  urlcolor={Aquamarine!70!black}
}
\usetikzlibrary{calc,backgrounds,positioning,bayesnet,decorations.pathmorphing,patterns,fit,hobby}
\theoremstyle{plain}

\theoremstyle{remark}
\newtheorem{remark}{Remark}
\theoremstyle{assumption}

\Crefname{algocf}{Algorithm}{Algorithms}
\crefname{algorithm}{Algorithm}{Algorithms}
\crefname{figure}{Figure}{Figure}
\crefname{section}{\S}{\S\S}
\Crefname{section}{\S}{\S\S}
\crefformat{equation}{(#2#1#3)}
\crefformat{enumi}{(#2#1#3)}
\glsdisablehyper{}
\newacronym{VAE}{vae}{variational autoencoder}
\newacronym{AE}{ae}{auto-encoder}
\newacronym{RBM}{rbm}{restricted Boltzmann machine}
\newacronym{KL}{kl}{Kullback-Leibler}
\newacronym{SGD}{sgd}{stochastic gradient descent}
\newacronym{SGA}{sga}{stochastic gradient ascent}
\newacronym{ELBO}{elbo}{evidence lower bound}
\newacronym[sort=beta]{BVAE}{\(\beta\)-vae}{}
\newacronym{TC}{tc}{total correlation}
\newacronym{MMD}{mmd}{maximum mean discrepancy}
\newacronym{GAN}{gan}{generative adversarial network}
\newacronym{DCGAN}{dcgan}{deep convolutional generative adversarial network}
\presetkeys{todonotes}{%
  backgroundcolor=blue!10!white,
  linecolor=blue!10!white,
  bordercolor=blue!10!white
}{}

\setdefaultleftmargin{4ex}{2ex}{}{}{}{}

\usepackage{scalerel}
\usepackage{mleftright}
\usepackage{dsfont}
\usepackage{bm}

\DeclareMathOperator{\E}{{}\mathbb{E}}

\DeclareMathOperator{\DKL}{\scalebox{0.95}{\text{KL}}}
\DeclareMathOperator*{\argmin}{arg\,min} 
\DeclareMathOperator*{\argmax}{arg\,max} 

\renewcommand{\to}{\ensuremath{\rightarrow}}              

\newcommand{\given}{\mid}
\newcommand{\grpP}[1]{\ensuremath{\mleft( #1 \mright)}}   
\newcommand{\grpS}[1]{\ensuremath{\mleft[ #1 \mright]}}   
\newcommand{\fnP}[2]{\ensuremath{#1 \grpP{#2}}}           
\newcommand{\fnS}[2]{\ensuremath{#1 \grpS{#2}}}           
\newcommand{\KL}[2]{\fnP{\DKL}{#1 \,\|\; #2}}             
\newcommand{\Ex}[2][]{\fnS{\E_{#1}}{#2}}                  
\newcommand{\dmt}[1]{\lvert #1 \rvert}

\newcommand{\p}[1]{\fnP{p_{{\theta}}}{#1}}
\newcommand{\q}[1]{\fnP{q_{{\phi}}}{#1}}

\newcommand{\gn}[1]{\fnP{\pi_{\theta,\beta}}{#1}}

\newcommand{\x}{\bm{x}}
\newcommand{\z}{\bm{z}}
\newcommand{\y}{\bm{y}}
\newcommand{\pz}{\fnP{p}{\z}}  
\newcommand{\fz}{f_{\beta}(\z)}
\newcommand{\Fz}{F_{\beta}}

\newcommand{\R}{\mathbb{R}}

\renewcommand{\L}{\ensuremath{\mathcal{L}}}

\newcommand\numberthis{\addtocounter{equation}{1}\tag{\theequation}}

\long\def\remark#1{
	\ifvmode\else
	\unskip\raisebox{-4.5pt}[0pt][0pt]{\rlap{$\scriptstyle\diamond$}}%
	\fi
	\setlength\marginparwidth{1.5cm}
	\marginpar{\raggedright\hbadness=10000
		\parindent=8pt \parskip=2pt
		\def\baselinestretch{0.8}\tiny
		\itshape\noindent #1\par}}
	
\usepackage[accepted]{icml2019}
\icmltitlerunning{Disentangling Disentanglement in Variational Autoencoders}

\begin{document}

\twocolumn[
\icmltitle{Disentangling Disentanglement in Variational Autoencoders}
\icmlsetsymbol{equal}{*}

\begin{icmlauthorlist}
  \icmlauthor{Emile Mathieu}{equal,oxstats}
  \icmlauthor{Tom Rainforth}{equal,oxstats}
  \icmlauthor{N. Siddharth}{equal,oxeng}
  \icmlauthor{Yee Whye Teh}{oxstats}
\end{icmlauthorlist}

\icmlaffiliation{oxstats}{Department of Statistics}  
\icmlaffiliation{oxeng}{Department of Engineering, University of Oxford}  

\icmlcorrespondingauthor{Emile Mathieu}{emile.mathieu@stats.ox.ac.uk}
\icmlcorrespondingauthor{Tom Rainforth}{rainforth@stats.ox.ac.uk}
\icmlcorrespondingauthor{N. Siddharth}{nsid@robots.ox.ac.uk}

\icmlkeywords{variational auto-encoders, disentanglement, representation learning, graphical models}

\vskip 0.3in
]



\printAffiliationsAndNotice{\icmlEqualContribution} 


\begin{abstract}
  We develop a generalisation of disentanglement in \glspl{VAE}---\emph{decomposition} of the latent representation---characterising it as the fulfilment of two factors:
  \begin{inparaenum}[a)]
  \item the latent encodings of the data having an appropriate level of overlap, and
  \item the aggregate encoding of the data conforming to a desired structure, represented through the prior.
  \end{inparaenum}
  Decomposition permits disentanglement, i.e. explicit independence between latents, as a special case, but also allows for a much richer class of properties to be imposed on the learnt representation, such as sparsity, clustering, independent subspaces, or even intricate hierarchical dependency relationships.
  We show that the \acrshort{BVAE} varies from the standard \gls{VAE} predominantly in its control of latent overlap and that for the standard choice of an isotropic Gaussian prior, its objective is invariant to rotations of the latent representation.
  Viewed from the decomposition perspective, breaking this invariance with simple manipulations of the prior can yield better disentanglement with little or no detriment to reconstructions.
  We further demonstrate how other choices of prior can assist in producing different decompositions and introduce an alternative training objective that allows the control of both decomposition factors in a principled manner.
\end{abstract}


\vspace*{-1.5\baselineskip}
\section{Introduction}
\label{sec:introduction}
An oft-stated motivation for learning disentangled representations of data with deep generative models is a desire to achieve interpretability~\citep{Bengio2013,Chen2016wm}---particularly the \emph{decomposability}~\citep[see \S 3.2.1 in][]{lipton2016mythos} of latent representations to admit intuitive explanations.
Most work has focused on capturing purely \emph{independent} factors of variation~\citep{Chen2016wm,DBLP:journals/corr/abs-1804-03599,Esmaeili2018up,Hyunjik2018,Ansari:2018tf,Xu:2018tl,alemi2016deep,chen2018isolating,higgins2016beta,eastwood2018a,ZhaoSE17b}, typically evaluating this using purpose-built, synthetic data~\citep{eastwood2018a,higgins2016beta,Hyunjik2018}, whose generative factors are independent by construction.

This conventional view of disentanglement, as recovering independence, has subsequently motivated the development of formal evaluation metrics for independence~\citep{eastwood2018a,Hyunjik2018}, which in turn has driven the development of objectives that target these metrics, often by employing regularisers explicitly encouraging independence in the representations~\citep{eastwood2018a,Hyunjik2018,Esmaeili2018up}.

We argue that such an approach is not generalisable, and potentially even harmful, to learning interpretable representations for more complicated problems, where such simplistic representations cannot accurately mimic the generation of high dimensional data from low dimensional latent spaces, and more richly structured dependencies are required.

We posit a generalisation of disentanglement in \glspl{VAE}---\emph{decomposing} their latent representations---that can help avoid such pitfalls.
We characterise decomposition in \glspl{VAE} as the fulfilment of two factors:
\begin{inparaenum}[a)]
\item the latent encodings of data having an appropriate level of overlap, and
\item the aggregate encoding of data conforming to a desired structure, represented through the prior.
\end{inparaenum}
We emphasize that neither of these factors is sufficient in isolation: without
an appropriate level of overlap, encodings can degrade to a lookup table
where the latents convey little information about data, and without the aggregate encoding of data following a desired structure, the encodings do not decompose as desired.

Disentanglement \emph{implicitly} makes a choice of decomposition: that the latent features are independent of one another.
We make this \emph{explicit} and exploit it to both provide improvement to disentanglement through judicious choices of structure in the prior, and to introduce a more general framework flexible enough to capture alternate, more complex, notions of decomposition such as sparsity, clustering, hierarchical structuring, or independent subspaces.

To connect our framework with existing approaches for encouraging disentanglement, we provide a theoretical analysis of the \acrshort{BVAE}~\citep{higgins2016beta,alemi2016deep,alemi2018fixing}, and show that it typically only allows control of latent overlap, the first decomposition factor.
We show that it can be interpreted, up to a constant offset, as the standard \gls{VAE} objective with its prior annealed as $\p{\z}^{\beta}$ and an additional maximum entropy regularization of the encoder that increases the stochasticity of the encodings.
Specialising this result for the typical choice of a Gaussian encoder and isotropic Gaussian prior indicates that the \acrshort{BVAE}, up to a scaling of the latent space, is equivalent to the \gls{VAE} plus a regulariser encouraging higher encoder variance.
Moreover, this objective is invariant to rotations of the learned latent representation, meaning that it does not, on its own, encourage the latent variables to take on meaningful representations any more than an arbitrary rotation of them.

We confirm these results empirically, while further using our decomposition framework to show that simple manipulations to the prior can improve disentanglement, and other decompositions, with little or no detriment to the reconstruction accuracy.
Further, motivated by our analysis, we propose an alternative objective that takes into account the distinct needs of the two factors of decomposition, and use it to learn clustered and sparse representations as demonstrations of alternative forms of decomposition.
An implementation of our experiments and suggested methods is provided at~\href{http://github.com/iffsid/disentangling-disentanglement}{http://github.com/iffsid/disentangling-disentanglement}.




\section{Background and Related Work}
\label{sec:background}

\subsection{Variational Autoencoders}

Let $\x$ be an $\mathcal{X}$-valued random variable distributed according to an unknown generative process with density $p_{\mathcal{D}}(\x)$ and from which we have observations, $X=\{\x_1, \dots, \x_n\}$.
The aim is to learn a latent-variable model $\p{\x,\z}$ that captures this generative process, comprising of a fixed\footnote{Learning the prior is possible, but omitted for simplicity.} prior over latents $p(\z)$ and a parametric likelihood $\p{\x | \z}$.
Learning proceeds by minimising a divergence between the true data generating distribution and the model w.r.t $\theta$, typically
\begin{align*}
  \argmin_{\bm{\theta}}
  \KL{p_\mathcal{D}(\x)}{\p{\x}}
  =
  \argmax_{\bm{\theta}}
  \Ex[p_\mathcal{D}(\x)]{\log \p{\x}}
\end{align*}
where \(\p{\x} = \int_{\mathcal{Z}} \p{\x|\z} p(\z) d\z\) is the marginal likelihood, or evidence, of datapoint~\(\x\) under the model, approximated by averaging over the observations.

However, estimating~\(\p{\x}\) (or its gradients) to any sufficient degree of accuracy is typically infeasible.
A common strategy to ameliorate this issue involves the introduction of a parametric inference model $\q{\z|\x}$ to construct a variational \gls{ELBO} on $\log \p{\x}$ as follows
\begin{align}
  \label{eq:elbo}
  \begin{split}
    \hspace*{-1ex}%
    \L(\x; \!\theta, \!\phi)
    \!&\triangleq\! \log \p{\x}-\KL{\q{\z|\x}}{\p{\z|\x}} \\
    \!&=\! \mathbb{E}_{\q{\z|\x}\!}[\log p_{{\theta}}(\x|\z)] \!-\! \KL{q_{{\phi}}(\z|\x)\!}{\!p(\z)\!}.\!\!\!
  \end{split}
\end{align}
A \acrfull{VAE}~\citep{KingmaW13,rezende2014stochastic} views this objective from the perspective of a deep stochastic autoencoder, taking the inference model $\q{\z|\x}$ to be an encoder and the likelihood model $\p{\x | \z}$ to be a decoder.
Here~\(\theta\) and~\(\phi\) are neural network parameters, and learning happens via \gls{SGA} using unbiased estimates of $\nabla_{\theta,\phi} \frac{1}{n} \sum_{i=1}^{n} \mathcal{L}(\x_i; {\theta}, {\phi})$.
Note that when clear from the context, we denote $\L(\x; \theta, \phi)$ as simply $\L(\x)$.

\subsection{Disentanglement}

Disentanglement, as typically employed in literature, refers to independence among features in a representation~\citep{Bengio2013,eastwood2018a,higgins2018towards}.
Conceptually, however, it has a long history, far longer than we could reasonably do justice here, and is far from specific to \glspl{VAE}.
The idea stems back to traditional methods such as ICA~\cite{yang1997adaptive,hyvarinen2000independent} and conventional autoencoders~\cite{schmidhuber1992learning}, through to a range of modern approaches employing deep learning~\cite{reed2014learning,makhzani2015adversarial,chen2016infogan,mathieu2016disentangling,achille19emergence,hjelm2018learning,cheung2014discovering}.

Of particular relevance to this work are approaches that explore disentanglement in the context of \glspl{VAE}~\cite{higgins2016beta,alemi2016deep,siddharth2017learning,Hyunjik2018,chen2018isolating,Esmaeili2018up}.
Here one aims to achieve independence between the dimensions of the aggregate encoding, typically defined as
\(
q_{\phi}(\z) \triangleq \E_{p_{\mathcal{D}}(\x)} \left[q(\z|\x)\right] \approx \frac{1}{n} \sum_{i}^n q(\z|\x_i)
\).
The significance of $q_{\phi}(\z)$ is that it is the marginal distribution induced on the latents by sampling a datapoint and then using the encoder to sample an encoding given that datapoint.  It can thus informally be thought of as the pushforward distribution for ``sampling'' representations in the latent space.

\begin{figure*}[t]
	\centering
	\includegraphics[width=.75\textwidth]{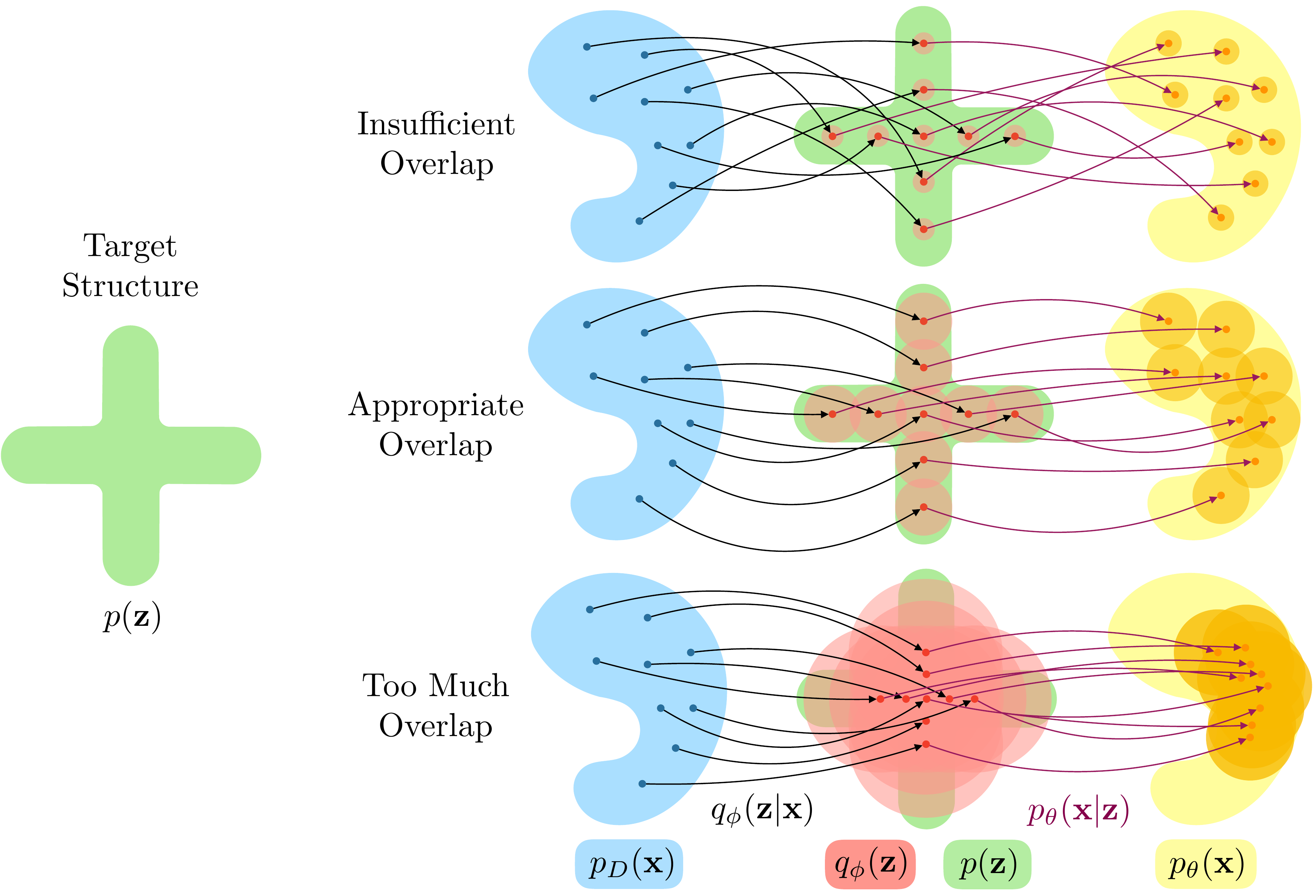}
	\caption{%
		Illustration of decomposition where the desired structure is a cross shape (enforcing \emph{sparsity}), expressed through the prior $p(\z)$ as shown on the left.
		In the scenario where there is insufficient overlap [top], we observe a lookup table behavior: points that are close in the data space are not close in the latent space and so the latent space loses meaning.
		In the scenario where there is too much overlap [bottom], the latent variable and observed datapoint convey little information about one another, such that the latent space again loses meaning.
		Note that if the distributional form of the latent distribution does not match that of the prior, as is the case here, this can also prevent the aggregate encoding matching the prior when the level of overlap is large.
			\label{fig:decomp-factors}
	}
	\vspace*{-1.8ex}
\end{figure*}

Within the disentangled \glspl{VAE} literature, there is also a distinction between unsupervised approaches, and semi-supervised approaches wherein one has access to the true generative factor values for some subset of data~\citep{kingma2014semi,siddharth2017learning,BouchacourtTN18}.
Our focus, however, is on the unsupervised setting.

Much of the prior work in the field has either implicitly or explicitly presumed a slightly more ambitious definition of disentanglement than considered above: that it is a measure of how well one captures \emph{true} factors of variation (which happen to be independent by construction for synthetic data), rather than just independent factors.
After all, if we wish for our learned representations to be interpretable, it is necessary for the latent variables to take on clear-cut meaning.

One such definition is given by \citet{eastwood2018a}, who define it as the extent to which a latent dimension~\(d \in D\) in a representation predicts a true generative factor~\(k \in K\), with each latent capturing at most one generative factor.
This implicitly assumes \(D \ge K\), as otherwise the latents are unable to explain all the true generative factors.
However, for real data, the association is more likely \(D \ll K\), with one learning a low-dimensional abstraction of a complex process involving many factors.
Consequently, such simplistic representations cannot, by definition,  be found for more complex datasets that require more richly structured dependencies to be able to encode the information required to generate higher dimensional data.
Moreover, for complex datasets involving a finite set of datapoints, it might not be reasonable to presume that one could capture the elements of the true generative process---the data itself might not contain sufficient information to recover these and even if it does, the computation required to achieve this through model learning is unlikely to be tractable.

The subsequent need for richly structured dependencies between latent dimensions has been reflected in the motivation for a handful of approaches~\citep{siddharth2017learning,BouchacourtTN18,Johnson:2016ud,Esmaeili2018up} that explore this through graphical models, although employing mutually-inconsistent, and not generalisable, interpretations of disentanglement.
This motivates our development of a decomposition framework as a means of extending beyond the limitations of disentanglement.




\section{Decomposition: A Generalisation of Disentanglement}
\label{sec:method}

The commonly assumed notion of disentanglement is quite restrictive for complex models where the true generative factors are not independent, very large in number, or where it cannot be reasonably assumed that there is a well-defined set of ``true'' generative factors (as will be the case for many, if not most, real datasets).
To this end, we introduce a generalization of disentanglement, \emph{decomposition}, which at a high-level can be thought of as imposing a desired structure on the learned representations.
This permits disentanglement as a special case, for which the desired structure is that \(\q{\z}\) factors along its dimensions.

We characterise the decomposition of latent spaces in \glspl{VAE} to be the fulfilment of two factors (as shown in \cref{fig:decomp-factors}):
%
\begin{compactenum}[a.]
\item An ``appropriate'' level of overlap in the latent space---ensuring that the range of latent values capable of encoding a particular datapoint is neither too small, nor too large.
This is, in general, dictated by the level of stochasticity in the encoder: the noisier the encoding process is, the higher the number of datapoints which can plausibly give rise to a particular encoding.
  \label{item:overlap}
  \item The aggregate encoding $\q{\z}$ matching the prior $\pz$, where the latter expresses the desired dependency structure between latents.
  \label{item:reg}
\end{compactenum}

The overlap factor~\cref{item:overlap} is perhaps best understood by considering extremes---too little, and the latents effectively become a lookup table; too much, and the data and latents do not convey information about each other.
In either case, meaningfulness of the latent encodings is lost.
Thus, without the \emph{appropriate} level of overlap---dictated both by noise in the true generative process and dataset size---it is not possible to enforce meaningful structure on the latent space.
Though quantitatively formalising overlap in general scenarios is surprisingly challenging (c.f.\ ~\cref{sec:overlap,sec:app-overlap}), we note for now that when the encoder distribution is unimodal, it is typically well-characterized by the mutual information between the data and the latents~\(I(\x;\z)\).

The regularisation factor~\cref{item:reg} enforces a congruence between the (aggregate) latent embeddings of data and the dependency structures expressed in the prior.
We posit that such structure is best expressed in the prior, as opposed to explicit independence regularisation of the marginal posterior~\citep{Hyunjik2018,chen2018isolating},
\begin{inparaenum}[]
\item to enable the \emph{generative} model to express the desired decomposition, and
\item to avoid potentially violating self-consistency between the encoder, decoder, and true data
generating distributions.
\end{inparaenum}
The prior also provides a rich and flexible means of expressing desired structure by defining a generative process that encapsulates dependencies between variables, as with a graphical model.

Critically, \emph{neither factor is sufficient in isolation}.
An inappropriate level of overlap in the latent space will impede interpretability, irrespective of quality of regularisation, as the latent space need not be meaningful.
Conversely, without the pressure to regularise to the prior, the latent space is under no constraint to exhibit the desired structure.

Decomposition is inherently subjective as we must choose the structure of the prior we regularise to depending on how we intend to use our learned model or what kind of features we would like to uncover from the data.
This may at first seem unsatisfactory compared to the seemingly objective adjustments often made to the \gls{ELBO} by disentanglement methods.
However, disentanglement \emph{is itself} a subjective choice for the decomposition.
We can embrace this subjective nature through judicious choices of the prior distribution; ignoring this imposes unintended assumptions which can have unwanted effects.
For example, as we will later show, the rotational invariance of the standard prior $p(\z)=\mathcal{N}(\z;0,I)$ can actually hinder disentanglement.




\section{Deconstructing the \acrshort{BVAE}}%
To connect existing approaches to our proposed framework, we now consider, as a case study, the \acrshort{BVAE}~\citep{higgins2016beta}---an adaptation of the \gls{VAE} objective (\acrshort{ELBO}) to learn better-disentangled representations.
Specifically, it scales the KL term in the standard ELBO by a factor $\beta > 0$ as
\begin{align}
  \hspace*{-1.4ex}%
  \L_\beta(\x)
  \!=\! \Ex[\q{\z|\x}\!]{\log \p{\x|\z}}
  \!-\! \beta \KL{\q{\z|\x}\!}{\!\pz\!}.\!\!
  \label{eq:beta-vae}
\end{align}
\citet{Hoffman2017uq} showed that the \acrshort{BVAE} target can be viewed as a standard \gls{ELBO} with the alternative prior \(r(\z) \propto \q{\z}^{(1-\beta)} p(\z)^{\beta}\), along with terms involving the mutual information and the prior's normalising constant.

We now introduce an alternate deconstruction as follows
\begin{restatable}{theorem}{betaThe}
  The $\beta$-VAE target~\(\L_\beta(\x)\)  can be interpreted in terms of the standard ELBO, $\mathcal{L}\left(\x;\pi_{\theta,\beta},q_{\phi}\right)$, for an adjusted target $\gn{\x,\z} \triangleq \p{\x \given \z} \fz$ with annealed prior \(\fz \triangleq \pz^{\beta}/\Fz\) as
  \begin{align}
    \mathcal{L}_{\beta}(\x)
    &= \mathcal{L}\left(\x;\pi_{\theta,\beta},q_{\phi} \right)
      + (\beta - 1) H_{q_{\phi}}
      + \log \Fz
      \label{eq:bvae}
  \end{align}
  where \(\Fz \triangleq \int_{\z} \pz^{\beta} d\z\) is constant given~\(\beta\), and \(H_{q_{\phi}}\) is the entropy of \(\q{\z \given \x}\).
\end{restatable}
\vspace*{-2ex}
\begin{proof}
	All proofs are given in~\cref{sec:bvae-thm}.
\end{proof}
\vspace*{-1ex}
Clearly, the second term in~\cref{eq:bvae}, enforcing a maximum entropy regulariser on the posterior~\(\q{\z \given \x}\), allows the value of~\(\beta\) to affect the overlap of encodings in the latent space.
We thus see that it provides a means of controlling decomposition factor (a).
However, it is itself not sufficient to enforce disentanglement.
For example, the entropy of $\q{\z \given \x}$ is independent of its mean $\mu_{\theta}(\x)$ and is independent to rotations of $\z$, so it is clearly incapable of discouraging certain representations with poor disentanglement.
All the same, having the wrong level of regularization can, in turn, lead to an inappropriate level of overlap and undermine the ability to disentangle.  Consequently, this term is still important.

Although the precise impact of prior annealing depends on the original form of the prior, the high-level effect is the same---larger values of $\beta$ cause the effective latent space to collapse towards the modes of the prior.
For uni-modal priors, the main effect of annealing is to reduce the scaling of $\z$; indeed this is the only effect for generalized Gaussian distributions.
While this would appear not to have any tangible effects, closer inspection suggests otherwise---it ensures that the scaling of the encodings matches that of the prior.
Only incorporating the maximum-entropy regularisation will simply cause the scaling of the latent space to increase.
The rescaling of the prior now cancels this effect, ensuring the scaling of $\q{\z}$ matches that of $p(\z)$.

Taken together, this implies that the \acrshort{BVAE}'s ability to encourage disentanglement is predominantly through \emph{direct} control over the level of overlap.
It places no other direct constraint on the latents to disentangle (although in some cases, the annealed prior may inadvertently encourage better disentanglement), but instead helps avoid the pitfalls of inappropriate overlap.
Amongst other things, this explains why large $\beta$ is not universally beneficial for disentanglement, as the level of overlap can be increased too far.

\subsection{Special Case -- Gaussians}
\label{sec:bvae-corr1}

We can gain further insights into the \acrshort{BVAE} in the common use case---assuming a Gaussian prior, $p(\z) = \mathcal{N}(\z;0,\Sigma)$, and Gaussian encoder, $\q{\z \given \x} = \mathcal{N}\left(\z ; \mu_{\phi}(\x),S_{\phi}(\x)\right)$.
Here it is straightforward to see that annealing simply scales the latent space by $1/\sqrt{\beta}$, i.e. $f_{\beta}(\z) = \mathcal{N}(\z;0,\Sigma/\beta)$.
Given this, it is easy to see that a \gls{VAE} trained with the adjusted target \(\mathcal{L}\left(\x ; \pi_{\theta,\beta},q_{\phi} \right)\), but appropriately scaling the latent space, will behave identically to one trained with the original target \(\mathcal{L}(\x)\).
It will also have an identical \gls{ELBO} as the expected reconstruction is trivially the same, while the \acrshort{KL} between Gaussians is invariant to scaling both equally.
More precisely, we have the following result.\vspace*{0.1ex}
\begin{restatable}{corollary}{gauss}
  \label{cor:gauss}
  If $\pz = \mathcal{N}(\z;0,\Sigma)$ and $\q{\z \given \x} = \mathcal{N}\left(\z ; \mu_{\phi}(\x),S_{\phi}(\x)\right)$, then,
  \begin{align}
    \label{eq:bvae3:app}
    \hspace*{-1.4ex}%
    \L_{\beta}(\x;\theta,\phi)
    &= \L\left(\x;\theta',\phi'\right)
    + \frac{(\beta - 1)}{2} \log \dmt{S_{\phi'}(\x)}
    + c
  \end{align}
  where~$\theta'$ and~$\phi'$ represent rescaled networks such that
  \begin{align*}
    p_{\theta'}(\x \given \z)
    &= \p{\x \given \z/\sqrt{\beta}}, \displaybreak[0] \\
    q_{\phi'}(\z | \x)
    &= \mathcal{N}\left(\z; \mu_{\phi'}(\x), S_{\phi'}(\x)\right), \displaybreak[0]  \\
    \mu_{\phi'}(\x)
    &=\sqrt{\beta}\mu_{\phi}(\x), \quad
    S_{\phi'}(\x)
    = \beta S_{\phi}(\x),
  \end{align*}
  and~\(c \triangleq \frac{D(\beta - 1)}{2} \left( 1 + \log \frac{2\pi}{\beta}\right) + \log \Fz\) is a constant, with~\(D\) denoting the dimensionality of~\(\z\).
\end{restatable}
Noting that as $c$ is irrelevant to the training process, this indicates an equivalence, up to scaling of the latent space, between training with the \acrshort{BVAE} objective and a maximum-entropy regularised version of the standard ELBO
\begin{align}
\label{eq:truth}
\mathcal{L}_{H,\beta}(\x)
\triangleq \mathcal{L}(\x)
+ \frac{(\beta - 1)}{2}\log \dmt{S_{\phi}(\x)},
\end{align}
whenever $\pz$ and $\q{\z \given \x}$ are Gaussian.
Note that we implicitly presume suitable adjustment of neural-network hyper-parameters and the stochastic gradient scheme to account for the change of scaling in the optimal networks.

Moreover, the stationary points for the two objectives \(\mathcal{L}_{\beta}(\x ; \theta, \phi)\) and \(\mathcal{L}_{H,\beta}\left(\x ; \theta',\phi' \right)\) are equivalent (c.f.\ Corollary 2 in \cref{sec:bvae-thm}), indicating that optimising for~\eqref{eq:truth} leads to networks equivalent
to those from optimising the \acrshort{BVAE} objective~\cref{eq:beta-vae}, up to scaling the encodings by a factor of $\sqrt{\beta}$.
Under the isotropic Gaussian prior setting, we further have the following result showing that the \acrshort{BVAE} objective is invariant to rotations of the latent space.

\begin{restatable}{theorem}{rotate}
  \label{the:rotate}
  If $\pz = \mathcal{N}(\z;0,\sigma I)$ and $\q{\z \given \x} = \mathcal{N}\left(\z ; \mu_{\phi}(\x),S_{\phi}(\x)\right)$, then for all rotation matrices $R$,
  \begin{align}
    \label{eq:bvae-rot}
    \L_{\beta}(\x;\theta,\phi)
    =& \L_{\beta}(\x;\theta^\dag(R),\phi^\dag(R))
  \end{align}
  where~$\theta^\dag(R)$ and~$\phi^\dag(R)$ are transformed networks such that
  \begin{align*}
    p_{\theta^\dag}(\x \given \z)
    &= \p{\x \given R^T\z}, \\
    q_{\phi^\dag}(\z | \x)
    &= \mathcal{N}\left(\z; R\mu_{\phi}(\x), R S_{\phi}(\x) R^T\right).
  \end{align*}
\end{restatable}

This shows that the \acrshort{BVAE} objective does not directly encourage latent variables to take on meaningful representations when using the standard choice of an isotropic Gaussian prior.
In fact, on its own, it encourages latent representations which match the true generative factors no more than it encourages \emph{any arbitrary rotation} of these factors, with such rotations capable of exhibiting strong correlations between latents.
This view is further supported by our empirical results (see~\cref{fig:disentenglement}), where we did not observe any gains in disentanglement (using the metric from \citet{Hyunjik2018}) from increasing $\beta > 0$ with an isotropic Gaussian prior trained on the \emph{2D Shapes} dataset~\citep{dsprites17}.
It may also go some way to explaining the extremely high levels of variation we found in the disentanglement-metric scores between different random seeds at train time.

It should be noted, however, that the value of $\beta$ can indirectly influence the level of disentanglement when using a mean-field assumption for the encoder distribution (i.e. restricting $S_{\phi}(x)$ to be diagonal).
As noted by~\citet{stuhmer2019isavae,rolinek2018variational}, increasing $\beta$ can reinforce existing inductive biases, wherein mean-field assumptions encourage representations which reduce dependence between the latent dimensions~\citep{turner2011two}.
%




\section{An Objective for Enforcing Decomposition}%
Given the characterisation set out above, we now develop an objective that incorporates the effect of both factors~(\labelcref{item:overlap}) and~(\labelcref{item:reg}).
Our analysis of the \acrshort{BVAE} tells us that its objective allows direct control over the level of overlap, i.e.\ factor~\cref{item:overlap}.
To incorporate direct control over the regularisation~\cref{item:reg} between the marginal posterior and the prior, we add a divergence term~\(\mathbb{D}(\q{z}, p(\z))\), yielding
\begin{align}
  \begin{split}
    \L_{\alpha, \beta}&(\x)
    = \Ex[\q{\z \given \x}]{\log \p{\x \given \z}}\\
    &- \beta~\KL{\q{\z \given \x}}{p(\z)}
    - \alpha~\mathbb{D}(\q{\z}, p(\z))
    \label{eq:alpha_obj}
  \end{split}
\end{align}
allowing control over how much factors (\labelcref{item:overlap}) and~(\labelcref{item:reg}) are enforced, through appropriate setting of $\beta$ and $\alpha$ respectively.

Note that such an additional term has been previously considered by \citet{Kumar:2017vs}, with $\mathbb{D}(\q{\z}, p(\z)) = \KL{\q{\z}}{p(\z)}$, although for the sake of tractability they rely instead on moment matching using covariances.
There have also been a number of approaches that decompose the standard \gls{VAE} objective in different ways \citep[e.g.~][]{Hoffman2016vz,Esmaeili2018up,dilokthanakul2019explicit} to expose \(\KL{\q{\z}}{p(\z)}\) as a component, but, as we discuss in \cref{sec:posteriorreg}, this can be difficult to compute correctly in practice, with common approaches leading to highly biased estimates whose practical behaviour is very different than the divergence they are estimating, unless very large batch sizes are used.

Wasserstein Auto-Encoders~\citep{Tolstikhin:2017wy} formulate an objective that includes a general divergence term between the prior and marginal posterior, computed using either \gls{MMD} or a variational formulation of the Jensen-Shannon divergence (a.k.a \acrshort{GAN} loss).
However, we find that empirically, choosing the \gls{MMD}'s kernel and numerically stabilising its U-statistics estimator to be tricky, and designing and learning a \acrshort{GAN} to be cumbersome and unstable.
Consequently, the problems of choosing an appropriate $\mathbb{D}(\q{\z}, p(\z))$ and generating reliable estimates for this choice are tightly coupled, with a general purpose solution remaining an important open problem; see further discussion in \cref{sec:posteriorreg}.



\begin{figure*}[t]
	\centering
	\includegraphics[width=0.45\textwidth]{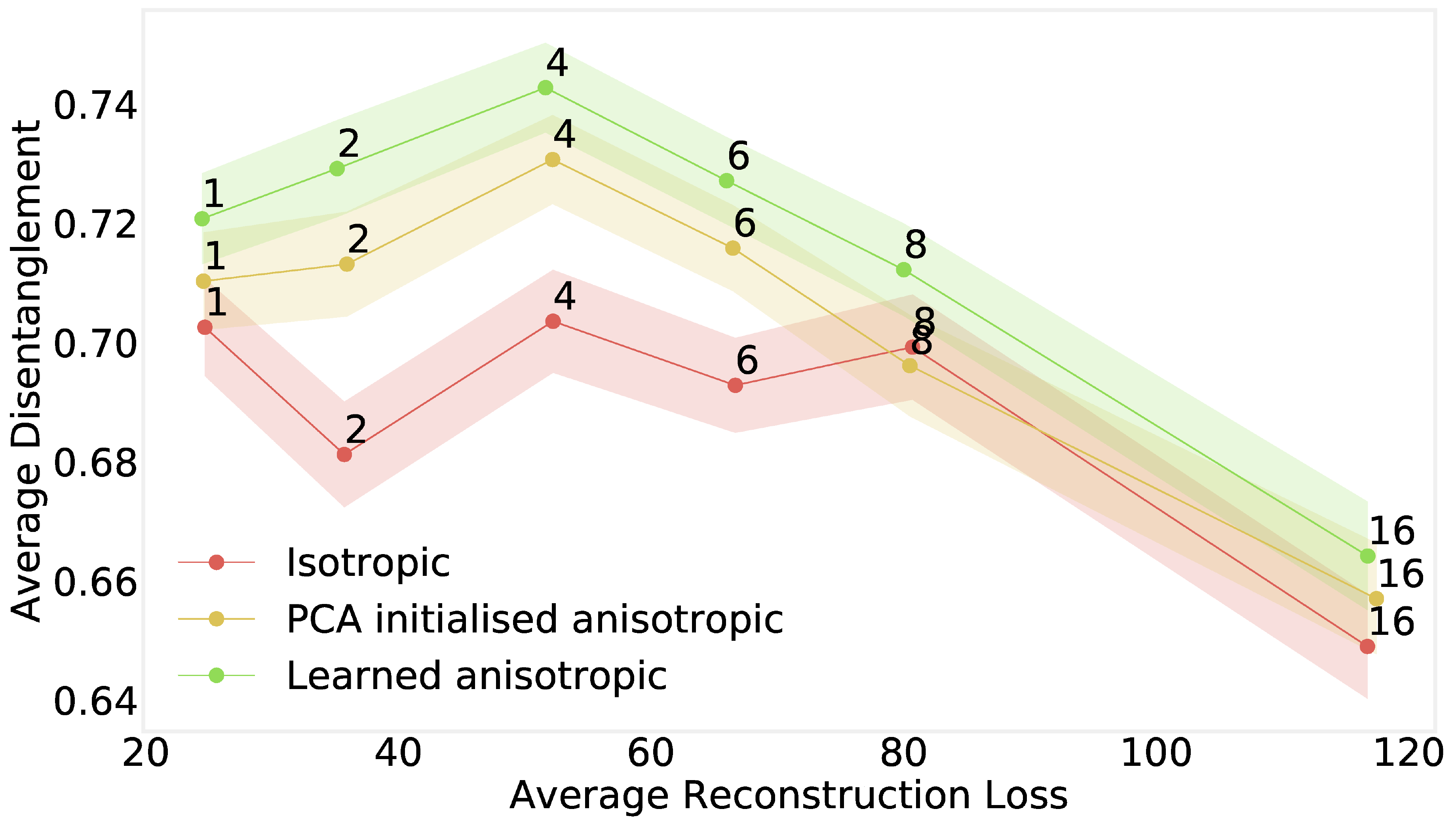} \quad
	\includegraphics[width=0.45\textwidth]{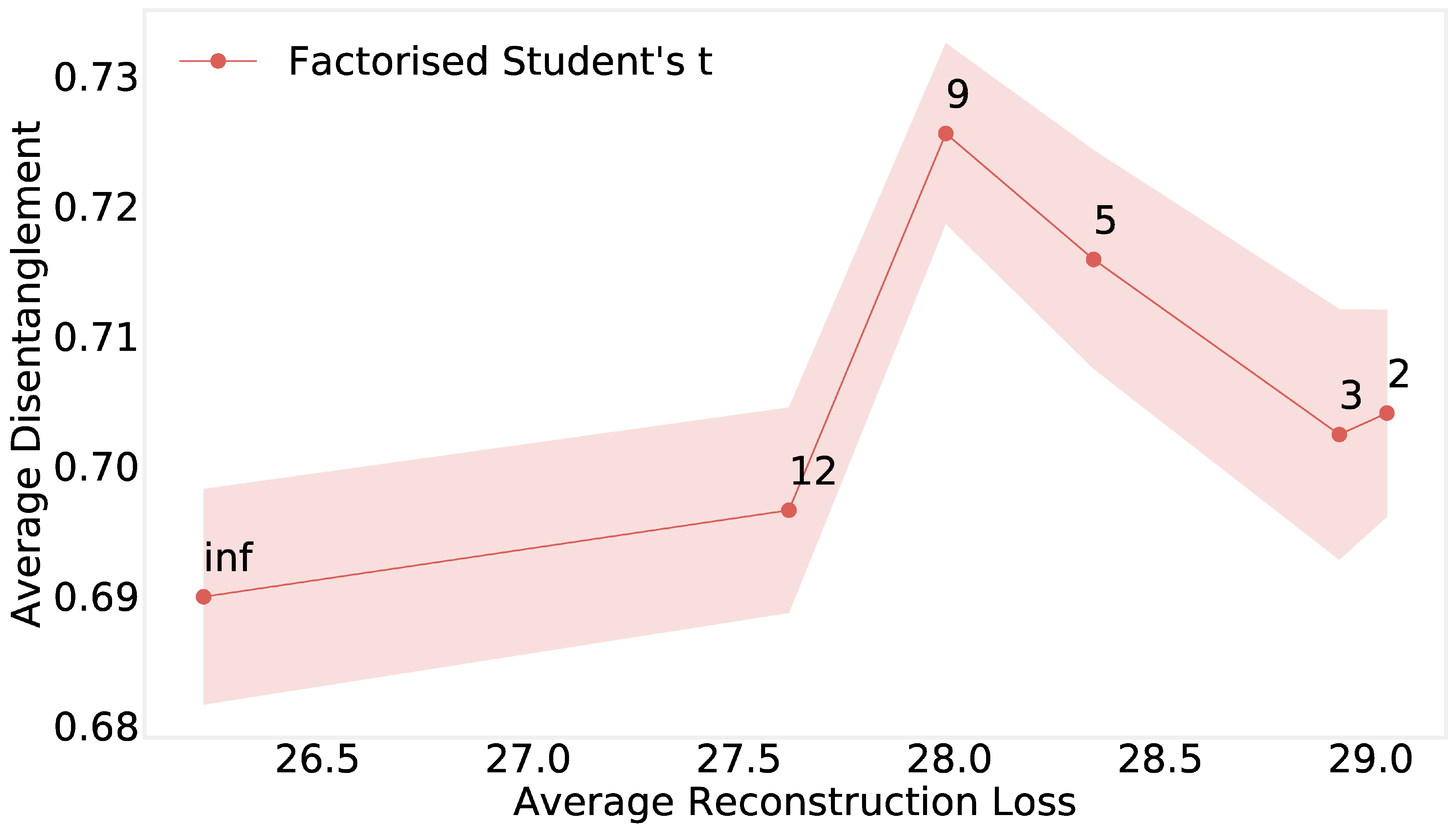}
	\caption{%
		Reconstruction loss vs disentanglement metric of~\citet{Hyunjik2018}.
		[Left] Using an anisotropic Gaussian with diagonal covariance either learned, or fixed to principal-component values of the dataset.
		Point labels represent different values of $\beta$.
		[Right] Using $p_{\nu}(\z) \!=\! \prod_{d} \!\textsc{Student-t}(z_d;\nu)$ for different $\nu$ with $\beta=1$.
		Note the different x-axis scaling.
		Shaded areas represent $\pm 2$ standard errors for estimated mean disentanglement calculated using $100$ separately trained networks.
		We thus see that the variability on the disentanglement metric is very large, presumably because of stochasticity in whether learned dimensions correspond to true generative factors.
		The variability in the reconstruction was only negligible and so is not shown.
		See \cref{sec:experimentaldetails} for full experimental details.
	}
	\label{fig:disentenglement}
\end{figure*}

\section{Experiments}
\label{sec:experiments}

\subsection{Prior for Axis-Aligned Disentanglement}%
We first show how subtle changes to the prior distribution can yield improvements in disentanglement.
The standard choice of an isotropic Gaussian has previously been justified by the correct assertion that the latents are independent under the prior~\citep{higgins2016beta}.
However, as explained in~\cref{sec:bvae-corr1}, the rotational invariance of this prior means that it does not directly encourage axis-aligned representations.
Priors that break this rotational invariance should be better suited for learning disentangled representations.
We assess this hypothesis by training a \acrshort{BVAE} (i.e.~\eqref{eq:alpha_obj} with $\alpha=0$) on the \emph{2D Shapes} dataset~\citep{dsprites17} and evaluating disentanglement using the metric of~\citet{Hyunjik2018}.

\Cref{fig:disentenglement} demonstrates that notable improvements in disentanglement can be achieved by using non-isotropic priors:
for a given reconstruction loss, implicitly fixed by $\beta$, non-isotropic Gaussian priors got better disentanglement scores, with further improvement achieved when the prior variance is learnt.
With a product of Student-t priors $p_{\nu}(\z)$ (noting $p_{\nu}(\z) \to \mathcal{N}(\z;\mathbf{0},\mathbf{I})$ as $\nu\to\infty$), reducing~\(\nu\) only incurred a minor reconstruction penalty, for improved disentanglement.
Interestingly,
very low values of~\(\nu\) caused the disentanglement score to drop again (though still giving higher values than the Gaussian).
We speculate that this may be related to the effect of heavy tails on the disentanglement metric itself, rather than being an objectively worse disentanglement.
Another interesting result was that for an isotropic Gaussian prior, as per the original \acrshort{BVAE} setup, no gains at all were achieved in disentanglement by increasing $\beta$.

\subsection{Clustered Prior}%
We next consider an alternative decomposition one might wish to impose---\emph{clustering} of the latent space.
For this, we use the ``pinwheels'' dataset from~\citep{Johnson:2016ud} and a mixture of four equally-weighted Gaussians as our prior.
We then conduct an ablation study to observe the effect of varying $\alpha$ and $\beta$ in $\mathcal{L}_{\alpha,\beta}(\mathbf{x})$ (as per~\eqref{eq:alpha_obj}) on the learned representations, taking the divergence to be $\text{KL}\left(p(\z) || \q{\z} \right)$ (see \cref{sec:experimentaldetails} for details).

\begin{figure}[t]
	\centering
	\begin{tabular}{@{}r@{\,}c@{}c@{}c@{}c@{}}
		& \(\beta = 0.01\)
		& \(\beta = 0.5\)
		& \(\beta = 1.0\)
		& \(\beta = 1.2\) \\
		\raisebox{4ex}{\rotatebox[origin=c]{90}{\(\alpha = 0\)}}
		& \includegraphics[width=0.235\columnwidth,trim={0 0em 9cm 0},clip]{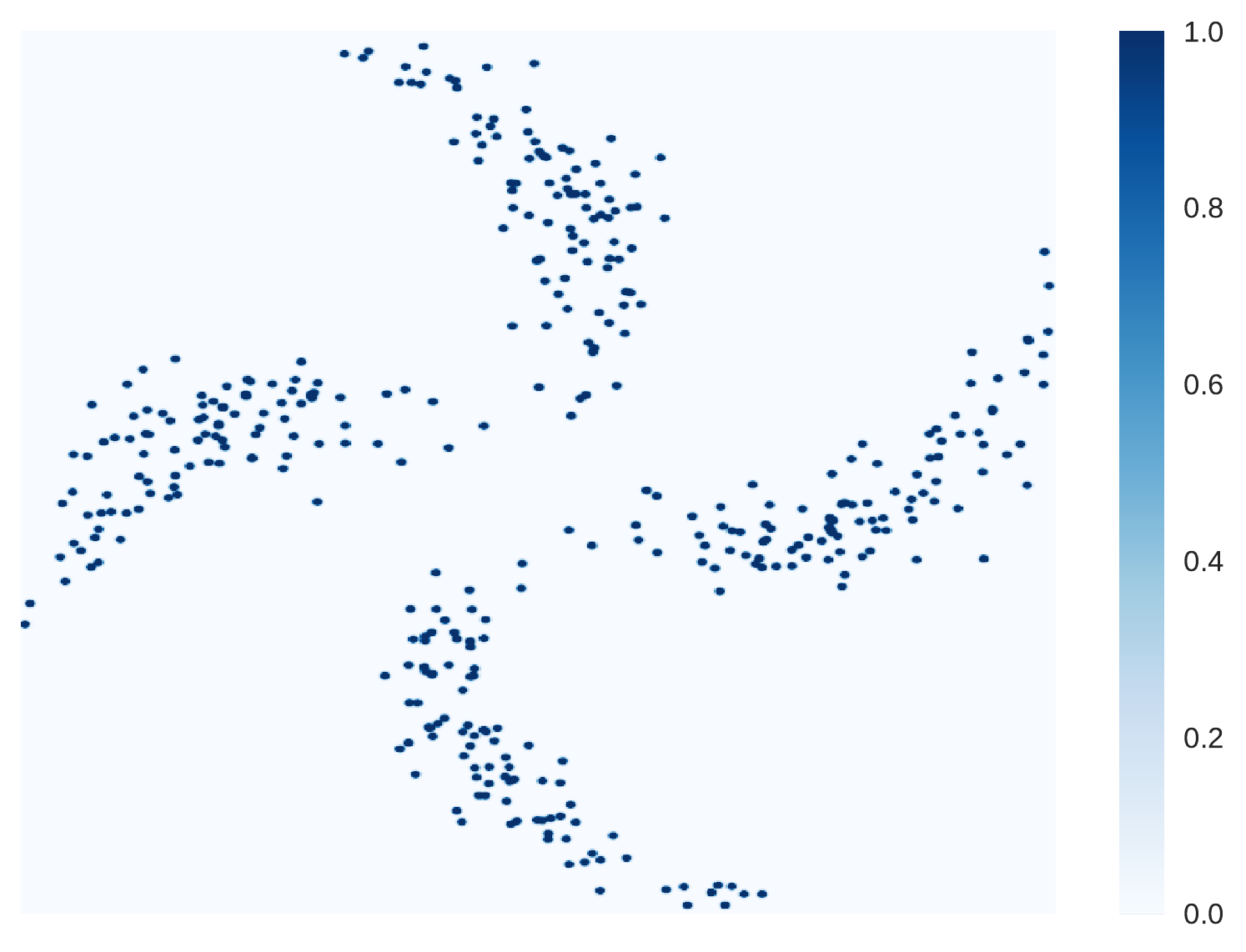}
		& \includegraphics[width=0.235\columnwidth,trim={0 0em 9cm 0},clip]{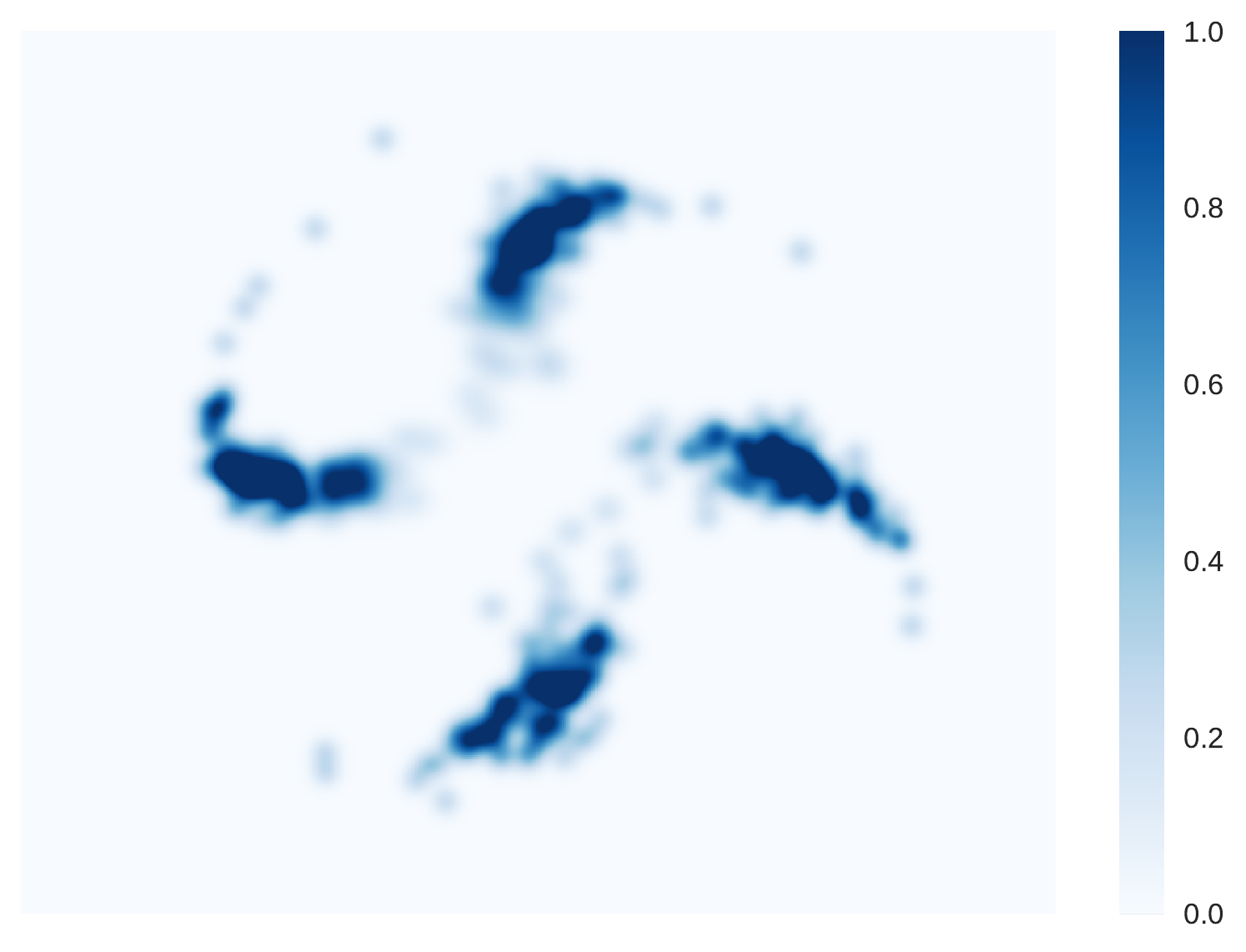}
		& \includegraphics[width=0.235\columnwidth,trim={0 0em 9cm 0},clip]{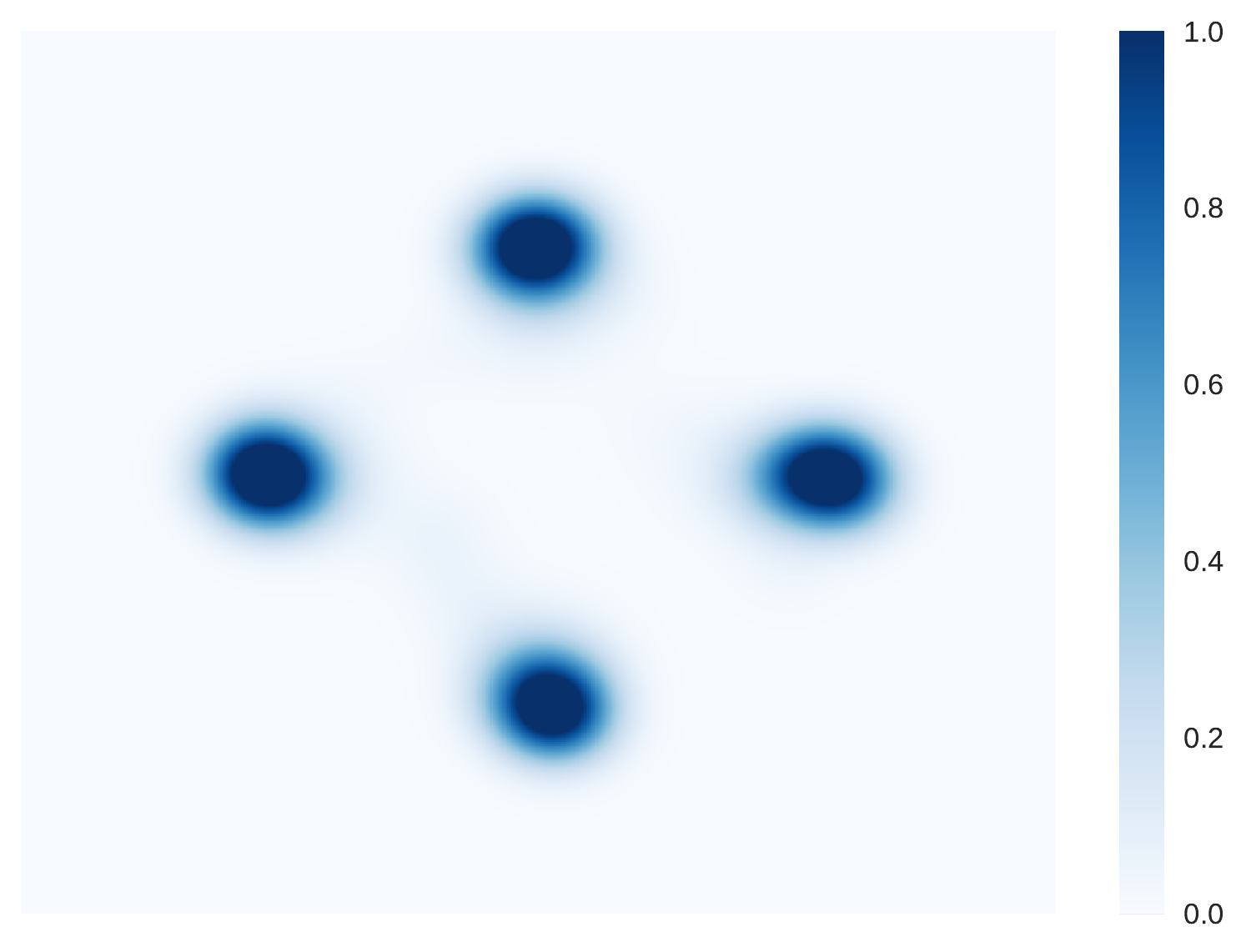}
		& \includegraphics[width=0.235\columnwidth,trim={0 0em 9cm 0},clip]{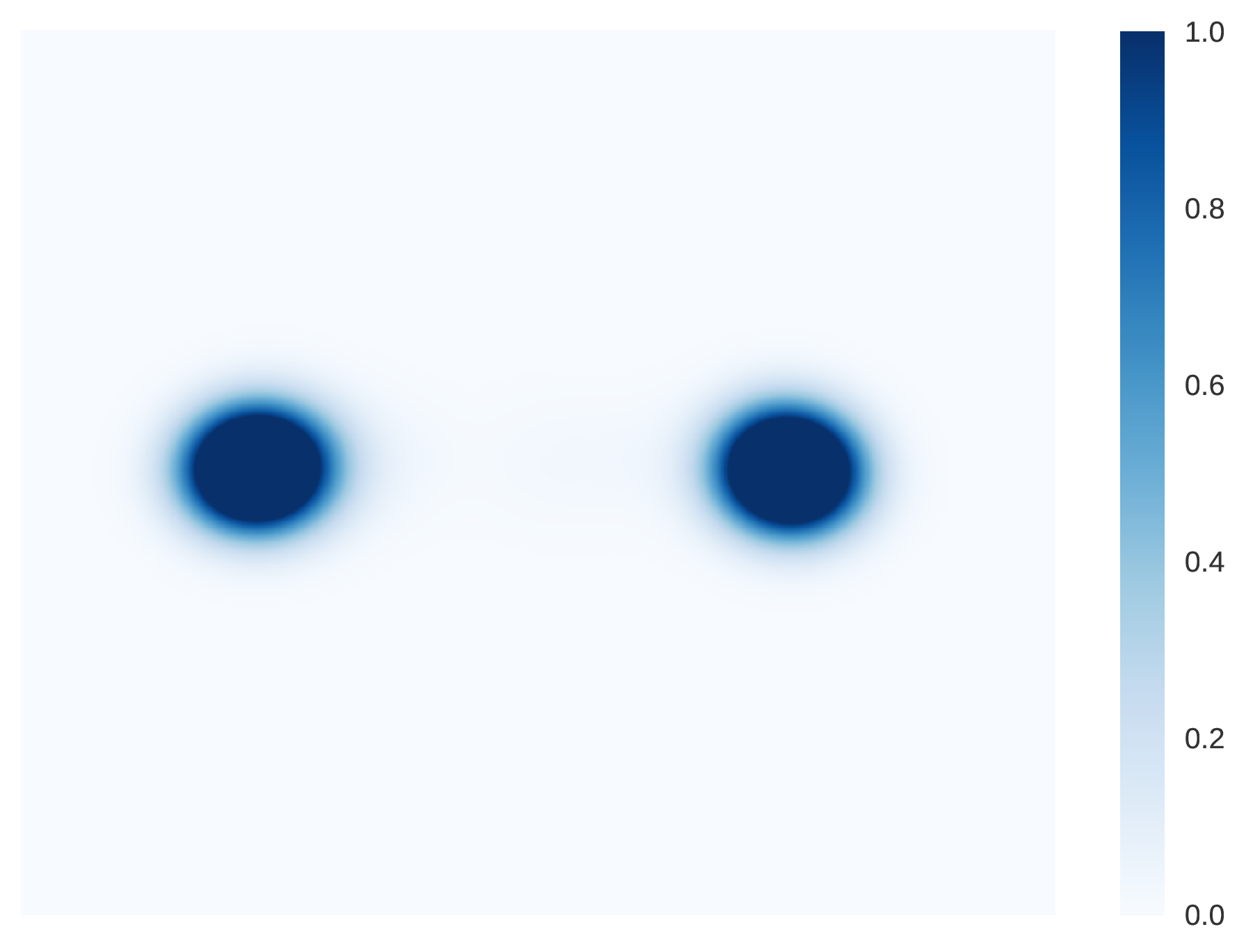} \\[-0.9ex]
		\raisebox{4ex}{\rotatebox[origin=c]{90}{\(\beta = 0\)}}\!
		& \includegraphics[width=0.235\columnwidth,trim={0 0em 9cm 0},clip]{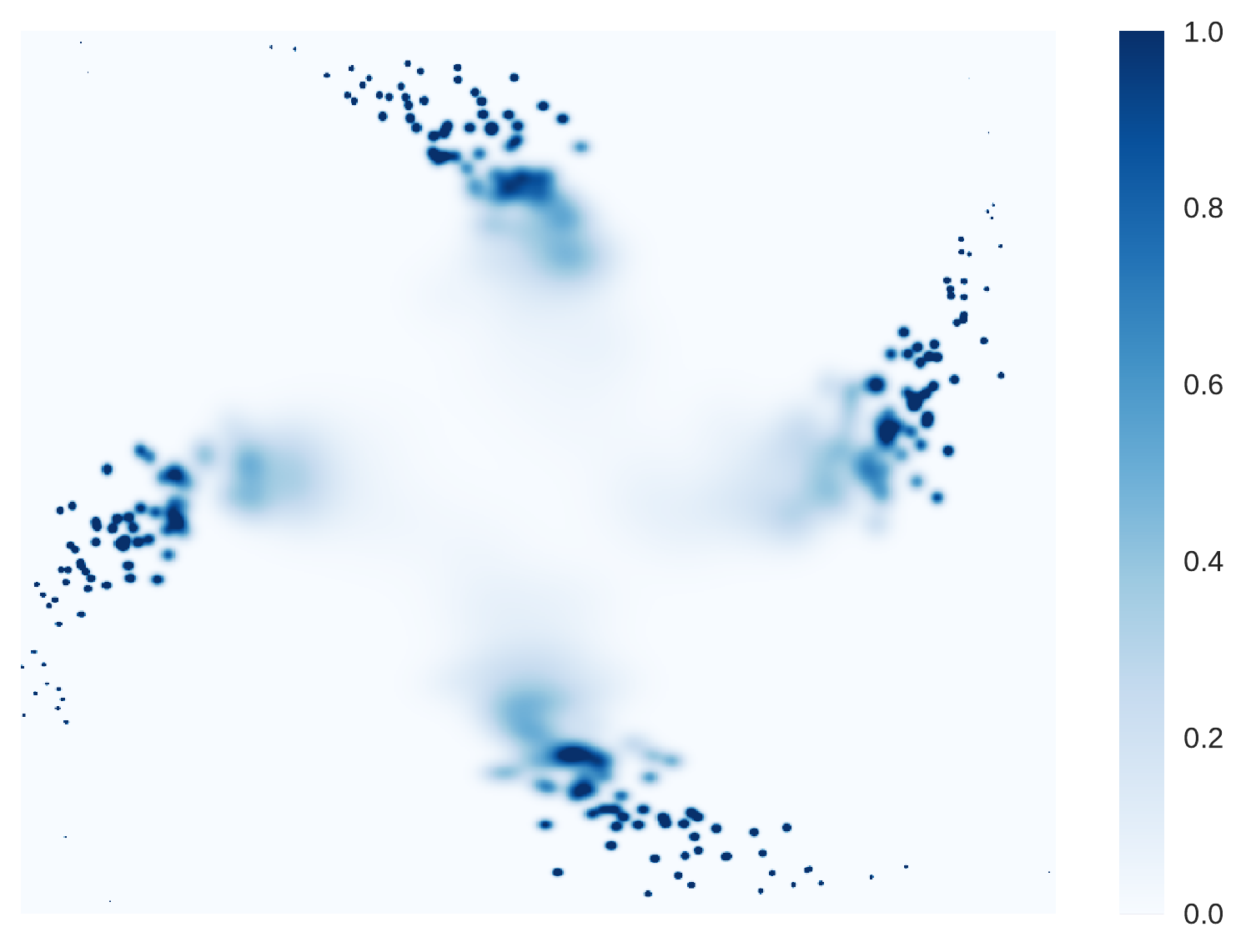}
		& \includegraphics[width=0.235\columnwidth,trim={0 0em 9cm 0},clip]{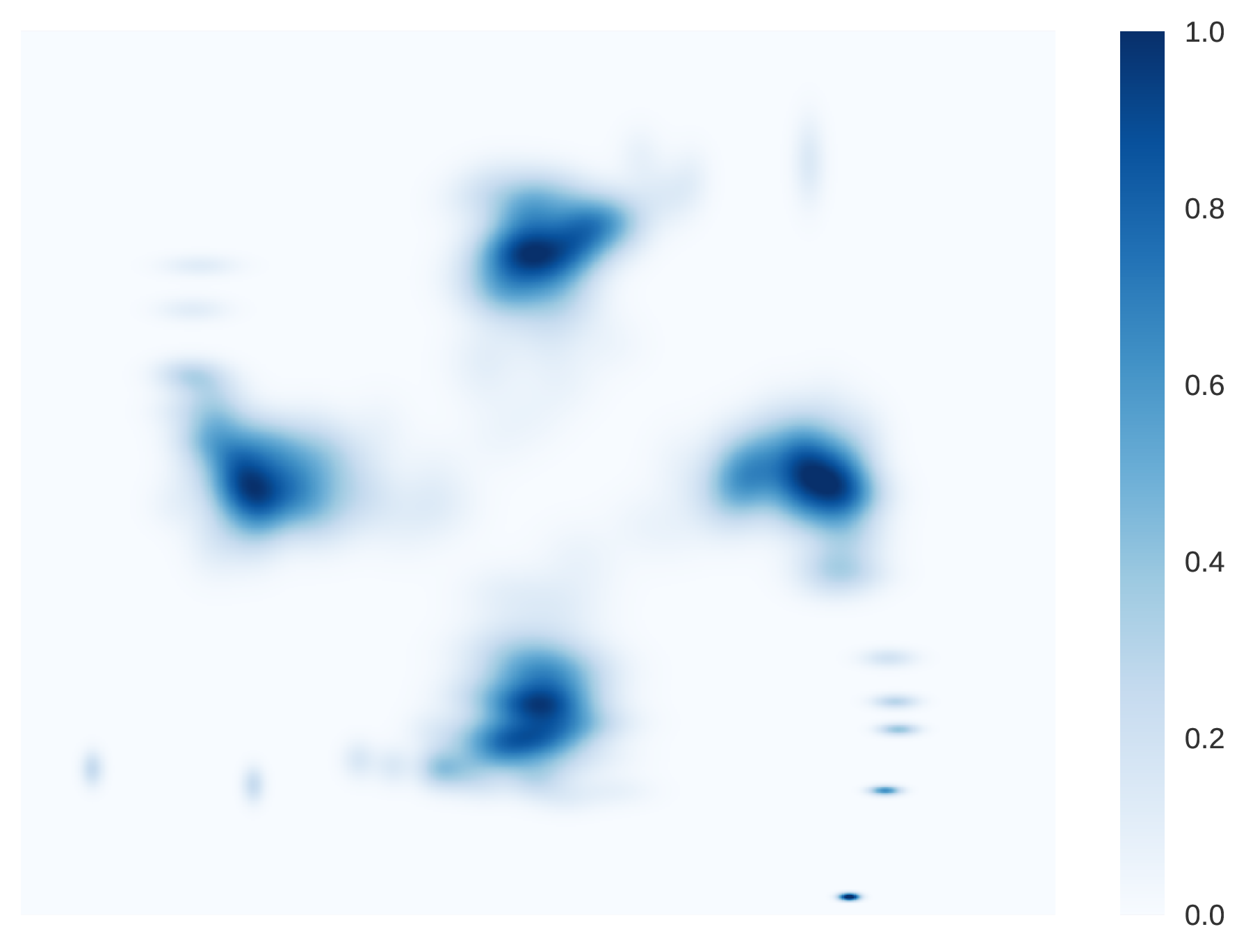}
		& \includegraphics[width=0.235\columnwidth,trim={0 0em 9cm 0},clip]{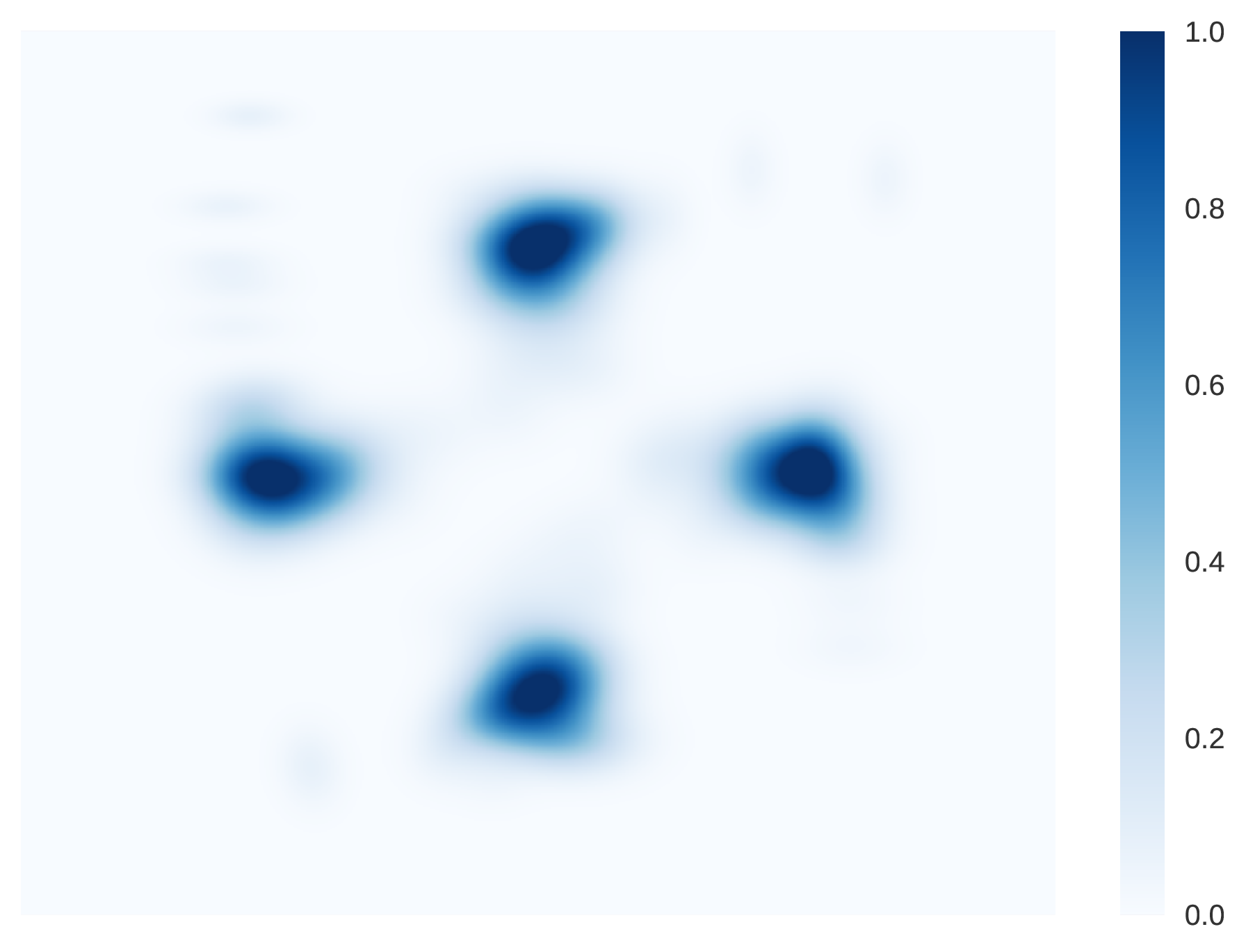}
		& \includegraphics[width=0.235\columnwidth,trim={0 0em 9cm 0},clip]{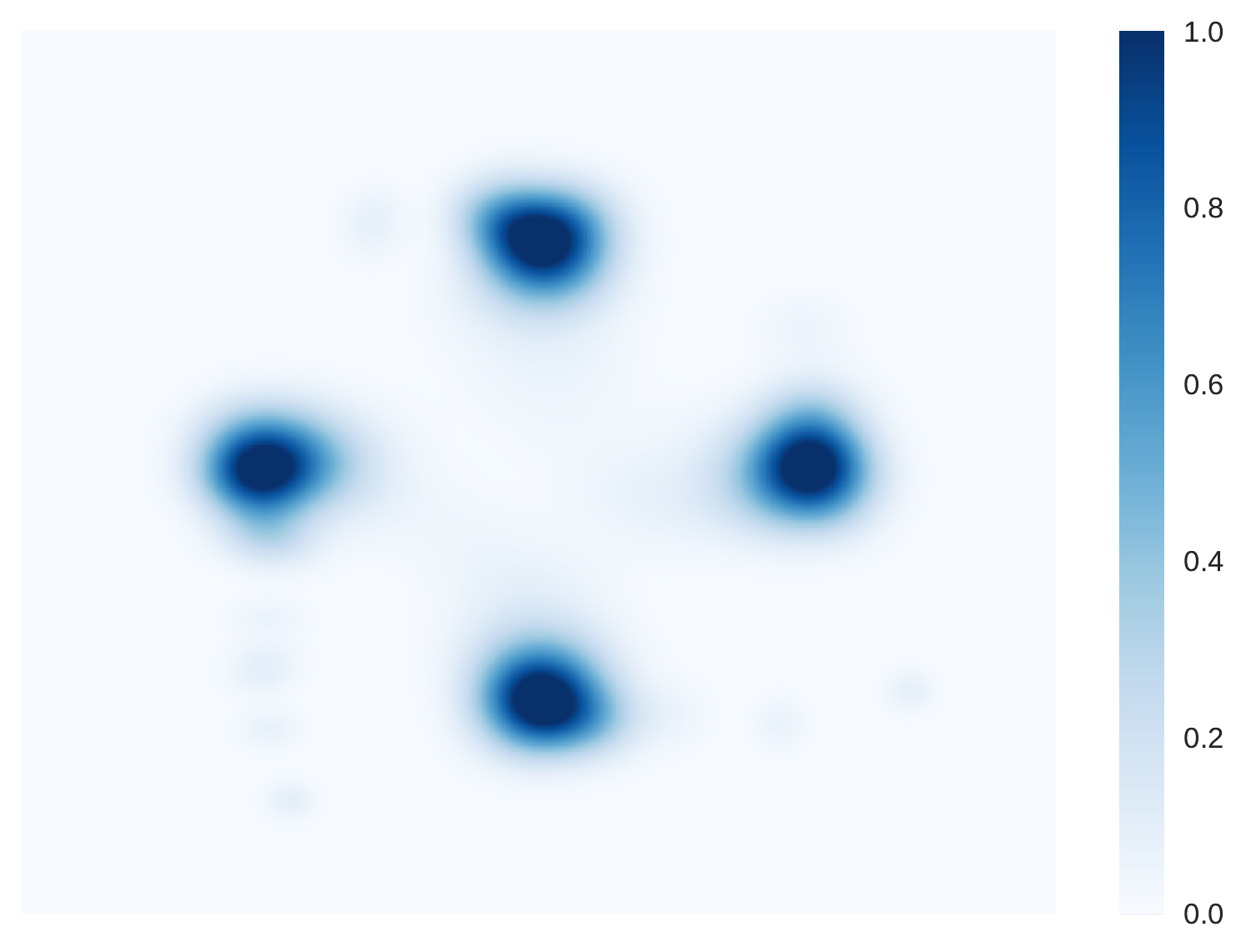} \\[-0.5ex]
		& \(\alpha = 1\)
		& \(\alpha = 3\)
		& \(\alpha = 5\)
		& \(\alpha = 8\)
	\end{tabular}
	\caption{%
		Density of aggregate posterior $\q{\z}$
		with different $\alpha$, $\beta$ for spirals dataset with a mixture of Gaussian prior.
	}
	\label{fig:mog}
\end{figure}

We see in \cref{fig:mog} that increasing $\beta$ increases the level of overlap in $\q{\z}$, as a consequence of increasing the encoder variance for individual datapoints.
When $\beta$ is too large, the encoding of a datapoint loses meaning.
Also, as a single datapoint encodes to a Gaussian distribution, $\q{\z|\x}$ is unable to match $\pz$ exactly.
Because $\q{\z|\x}\to\q{\z}$ when $\beta\to\infty$, this in turn means that
overly large values of $\beta$ actually cause a mismatch between  $\q{\z}$ and $\pz$ (see top right of \cref{fig:mog}).
Increasing \(\alpha\), instead always improved the match between $\q{\z}$ and $\pz$.
Here, the finiteness of the dataset and the choice of divergence results in an increase in overlap with increasing $\alpha$, but only up to the level required for a non-negligible overlap between the nearby datapoints: large values of $\alpha$ did not cause the encodings to collapse to a mode.

\begin{figure*}[t]
  \centering
  \includegraphics[width=0.305\textwidth]{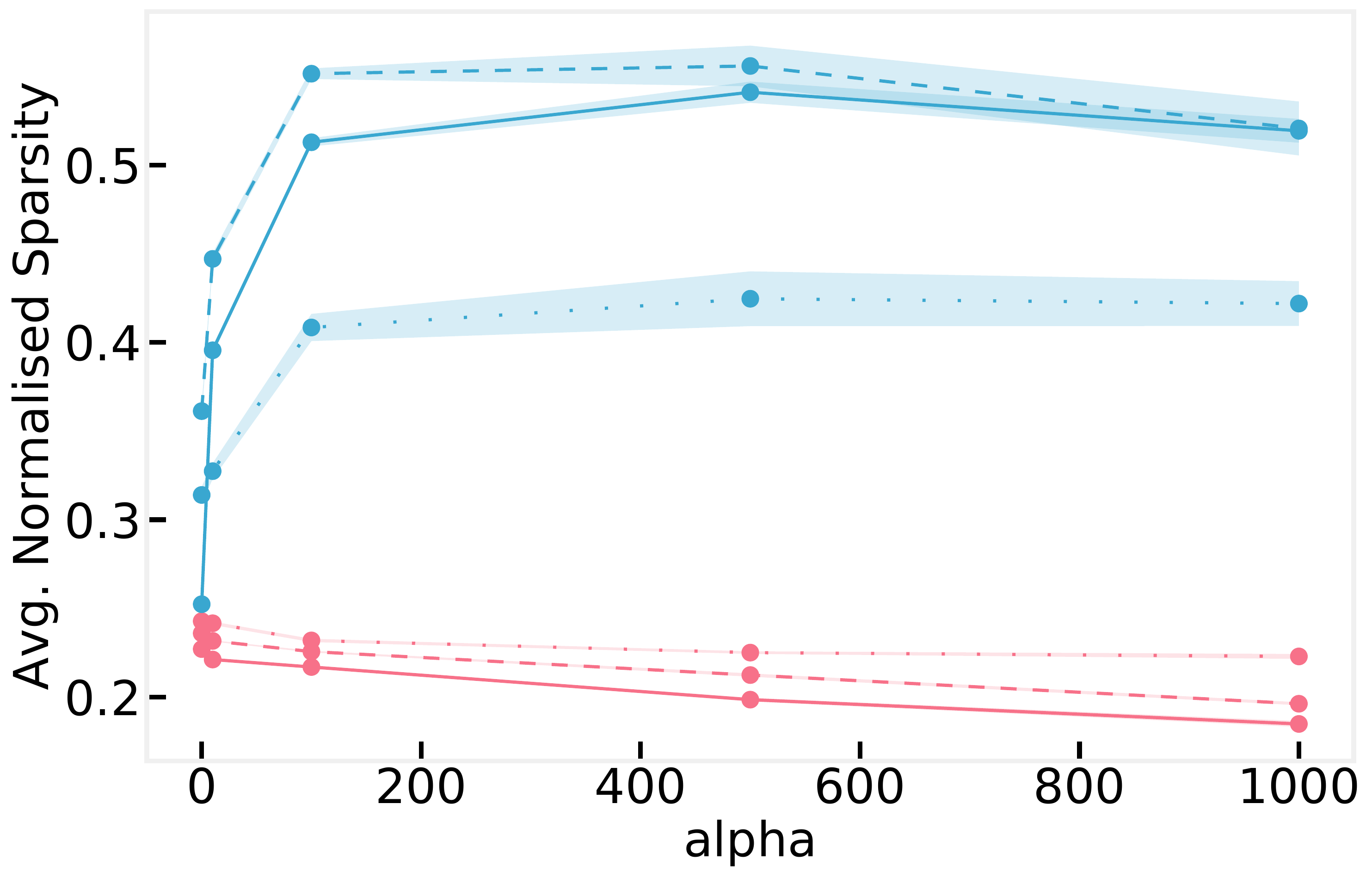}
  ~~~\includegraphics[width=0.305\textwidth]{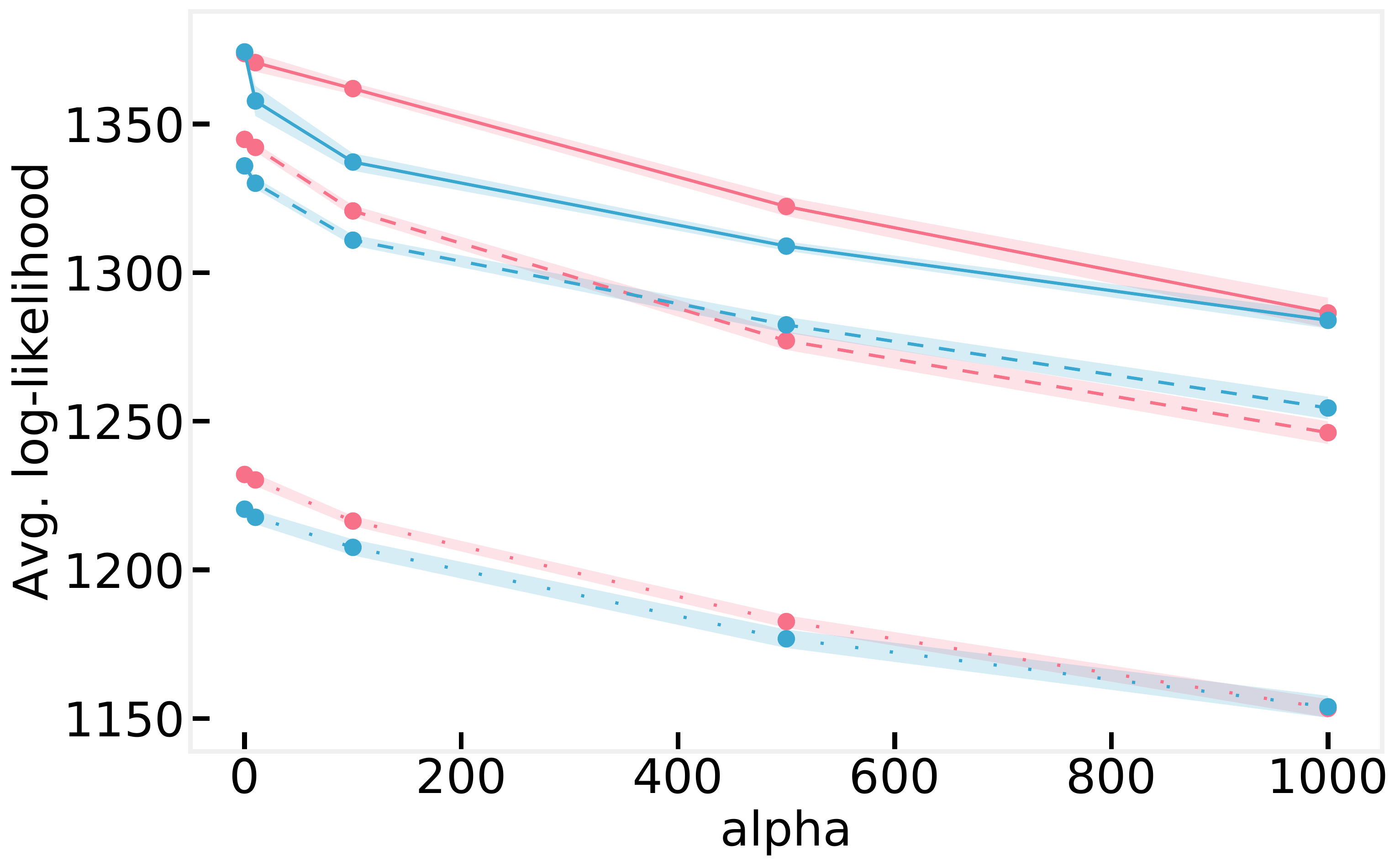}
  ~~~\includegraphics[width=0.305\textwidth]{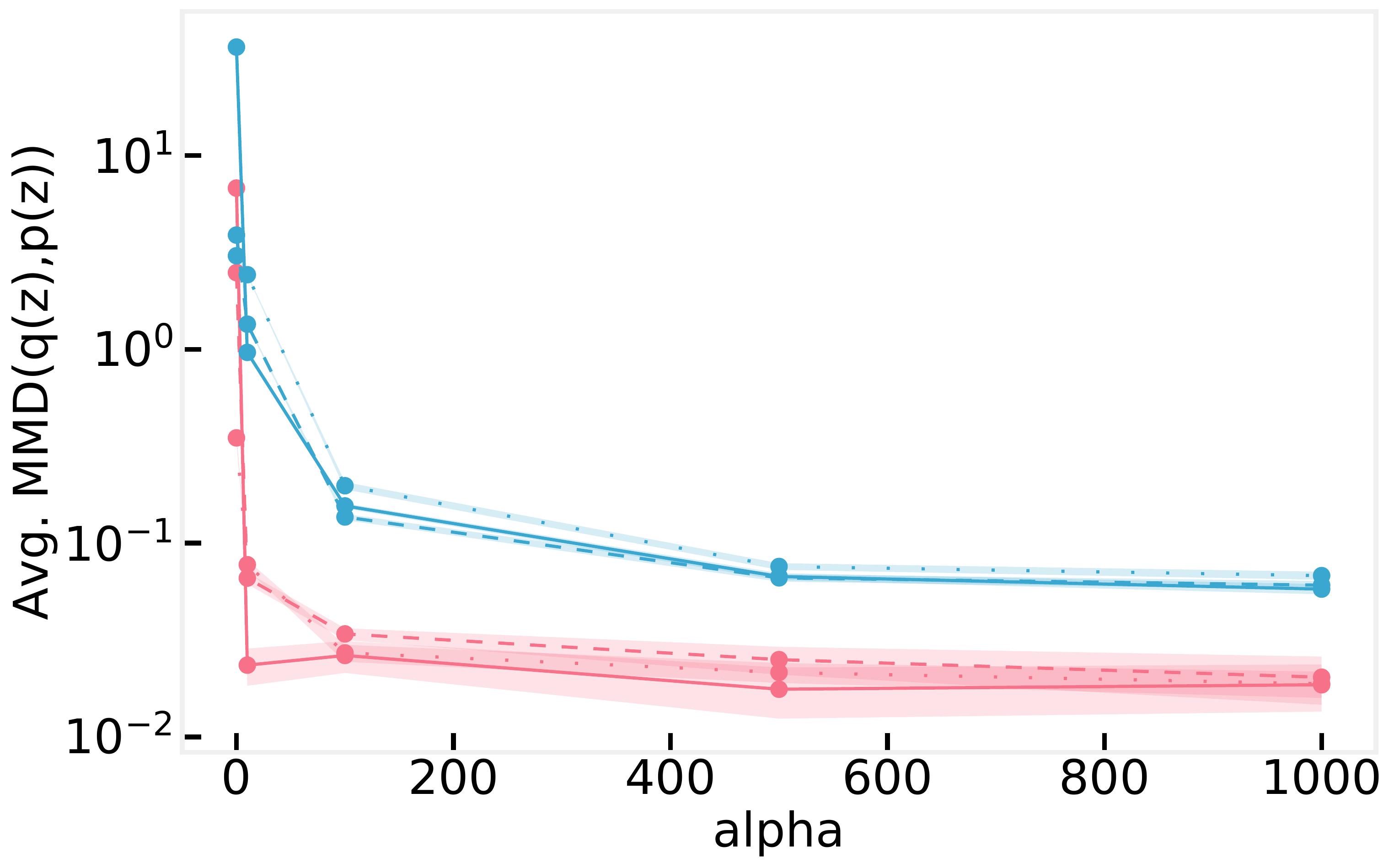}
  \qquad \includegraphics[width=0.35\textwidth]{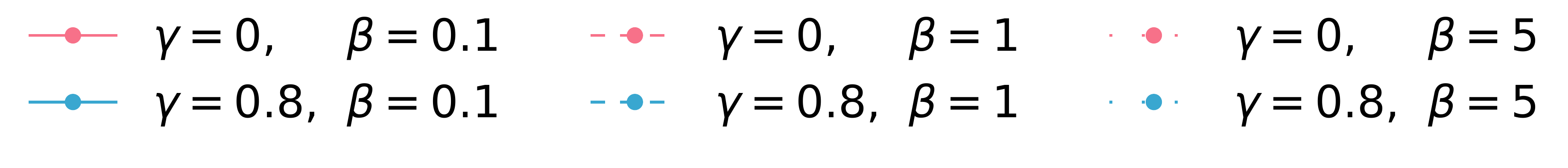}
  \caption{%
    [Left] Sparsity vs regularisation strength $\alpha$ (c.f.\ \cref{eq:alpha_obj}, high better).
    [Center] Average reconstruction log-likelihood $\E_{p_D(\x)}[\E_{\q{\z|\x}}[\log \p{\x | \z}]]$ vs  $\alpha$ (higher better).
    [Right] Divergence (\gls{MMD}) vs $\alpha$ (lower better).
    Note here that the different values of $\gamma$ represent regularizations to different distributions, with regularization to a Gaussian (i.e. $\gamma=0$) much easier to achieve than the sparse prior, hence the lower divergence.
    Shaded areas represent $\pm2$ standard errors in the mean estimate calculated using $8$ separately trained networks. 
    See \cref{sec:experimentaldetails} for full experimental details.
  }
  \label{fig:sparsity}
\end{figure*}

\subsection{Prior for Sparsity}%

Finally, we consider a commonly desired decomposition---sparsity, which stipulates that only a small fraction of available factors are employed.
That is, a \emph{sparse representation} \citep{olshausen96} can be thought of as one where each embedding has a significant proportion of its dimensions \emph{off}, i.e. close to $0$.
Sparsity has often been considered for feature-learning \citep{Larochelle:2008:CUD:1390156.1390224,conf/icml/CoatesN11} and employed in the probabilistic modelling literature \citep{razanto2007, SparseCoding}.

Common ways to achieve sparsity are through a specific penalty (e.g. $l_1$) or a careful choice of prior (peaked at 0).
Concomitant with our overarching desire to encode requisite structure in the prior, we adopt the latter, constructing a sparse prior as
$
p(\z) = \prod\nolimits_{d} ~(1 - \gamma) ~\mathcal{N}(z_d;0, 1) + \gamma ~\mathcal{N}(z_d;0, \sigma_0^2)$
with $\sigma_0^2=0.05$.
This mixture distribution can be interpreted as a mixture of samples being either \emph{off} or \emph{on}, whose proportion is set by the weight parameter $\gamma$.
We use this prior to learn a \gls{VAE} for the \emph{Fashion-MNIST} dataset~\citep{xiao2017/online} using the objective $\mathcal{L}_{\alpha,\beta}(\mathbf{x})$ (as per~\eqref{eq:alpha_obj}), taking the divergence to be an \gls{MMD} with a kernel that only considers difference between the marginal distributions (see \cref{sec:experimentaldetails} for details).

We measure a representation's sparsity using the \emph{Hoyer} extrinsic metric~\citep{Hurley2008ComparingMO}.
For $\y \in \R^d$,
\begin{align*}
\text{Hoyer}~(\y) = \frac{\sqrt{d} - \| \y \|_1 / \| \y \|_2}{\sqrt{d} - 1} \in [0, 1],
\end{align*}
yielding $0$ for a fully dense vector and $1$ for a fully sparse vector.
Rather than employing this metric directly to the mean encoding of each datapoint, we first normalise each dimension to have a standard deviation of $1$ under its aggregate distribution, i.e. we use $\bar{z}_d = z_d / \sigma(z_d)$ where $\sigma(z_d)$ is the standard deviation of dimension $d$ of the latent encoding taken over the dataset.
This normalisation is important as one could achieve a ``sparse'' representation simply by having different dimensions vary along different length scales (something the \acrshort{BVAE} encourages through its pruning of dimensions~\citep{stuhmer2019isavae}), whereas we desire a representation where different datapoints ``activate'' different features.
We then compute overall sparsity by averaging over the dataset as \(\text{Sparsity} = \frac{1}{n} \sum\nolimits_{i}^n \text{Hoyer}~(\bar{\z}_i)\).
\cref{fig:sparsity} (left) shows that substantial sparsity can be gained by replacing a Gaussian prior ($\gamma=0$) by a sparse prior ($\gamma=0.8$).
It further shows substantial gains from the inclusion of the aggregate posterior regularization, with $\alpha=0$ giving far low sparsity than $\alpha>0$, when using our sparse prior.
The use of our sparse prior did not generally harm the reconstruction compared.
Large values of $\alpha$ did slightly worsen the reconstruction, but this drop-off was much slower than increases in $\beta$ (note that $\alpha$ is increased to much higher levels than $\beta$).
Interestingly, we see that $\beta$ being either too low or too high also harmed the sparsity.

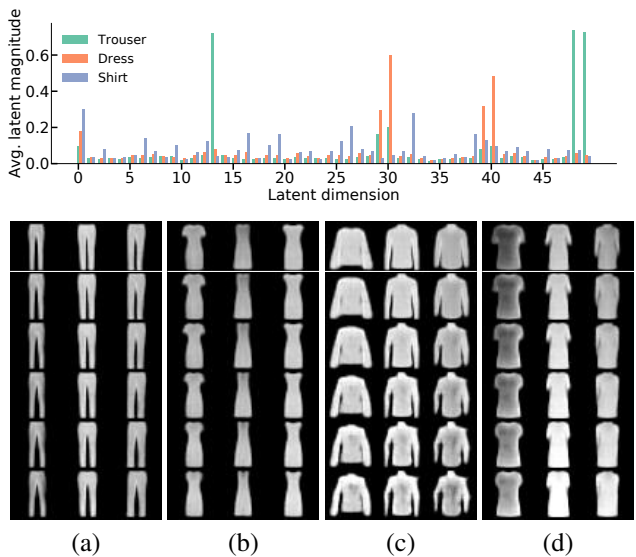
\begin{figure}[t]
  \centering
  \includegraphics[width=\linewidth]{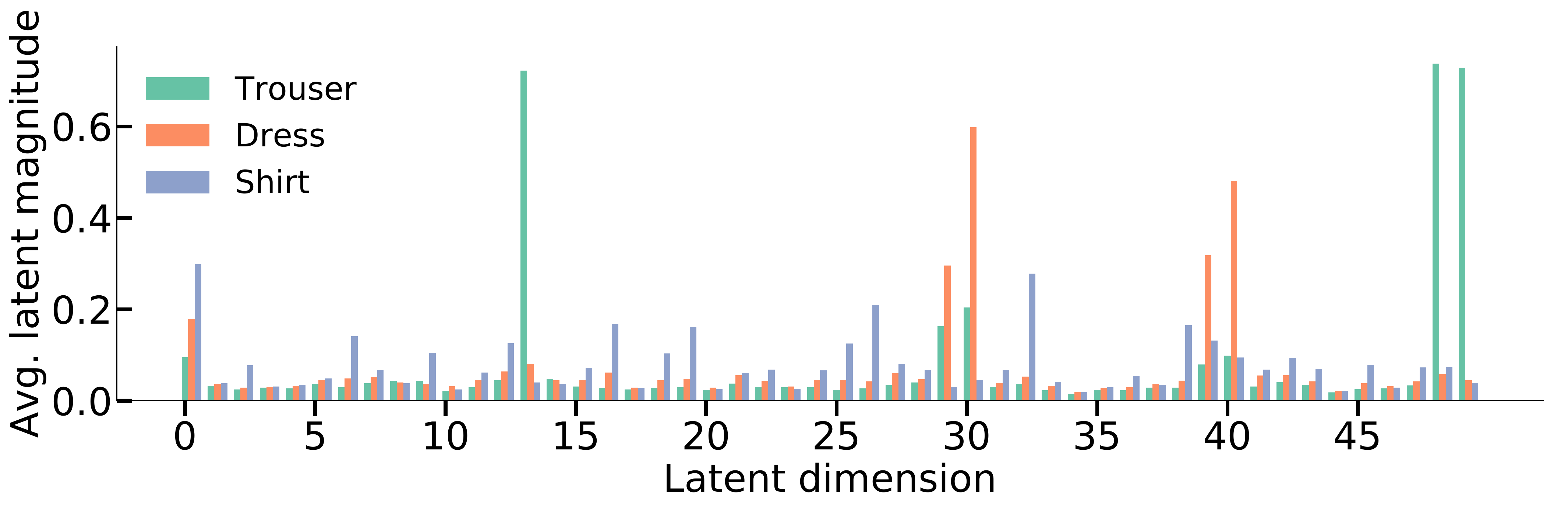}\\[1ex]
  \begin{tabular}[b]{@{\,}c@{\,}c@{\,}c@{\,}c@{}}
    \begin{tikzpicture}
      \node[anchor=south west,inner sep=0] (image) at (0,0) {%
        \includegraphics[width=0.243\linewidth]{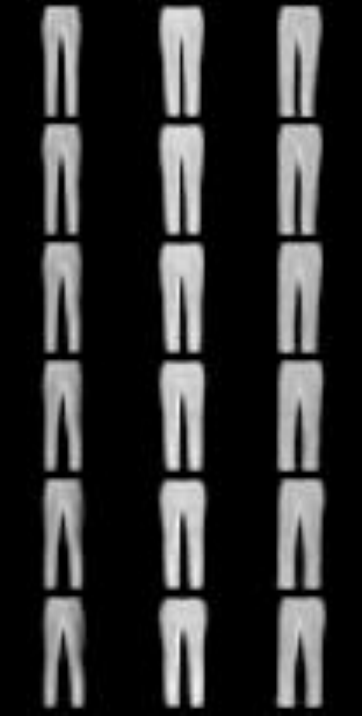}
      };
      \begin{scope}[x={(image.south east)},y={(image.north west)}]
        \draw[white] (0, 0.83) -- (1, 0.83);
      \end{scope}
    \end{tikzpicture}
    &
    \begin{tikzpicture}
      \node[anchor=south west,inner sep=0] (image) at (0,0) {%
        \includegraphics[width=0.243\linewidth]{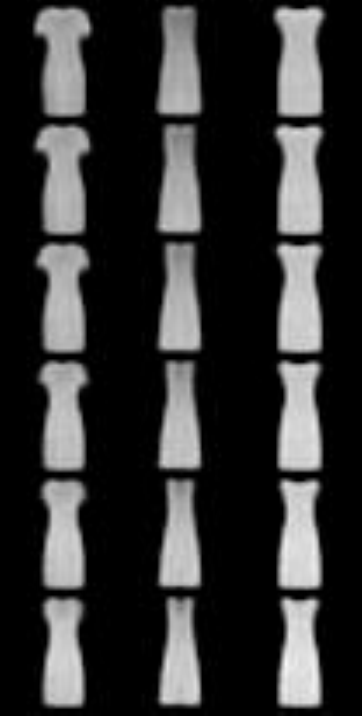}
      };
      \begin{scope}[x={(image.south east)},y={(image.north west)}]
        \draw[white] (0, 0.83) -- (1, 0.83);
      \end{scope}
    \end{tikzpicture}
    &
    \begin{tikzpicture}
      \node[anchor=south west,inner sep=0] (image) at (0,0) {%
        \includegraphics[width=0.243\linewidth]{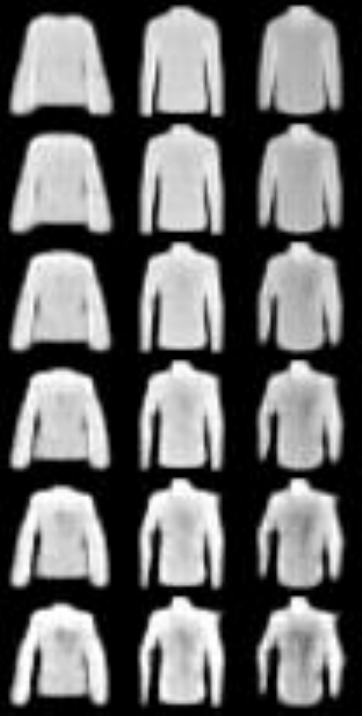}
      };
      \begin{scope}[x={(image.south east)},y={(image.north west)}]
        \draw[white] (0, 0.83) -- (1, 0.83);
      \end{scope}
    \end{tikzpicture}
    &
    \begin{tikzpicture}
      \node[anchor=south west,inner sep=0] (image) at (0,0) {%
        \includegraphics[width=0.243\linewidth]{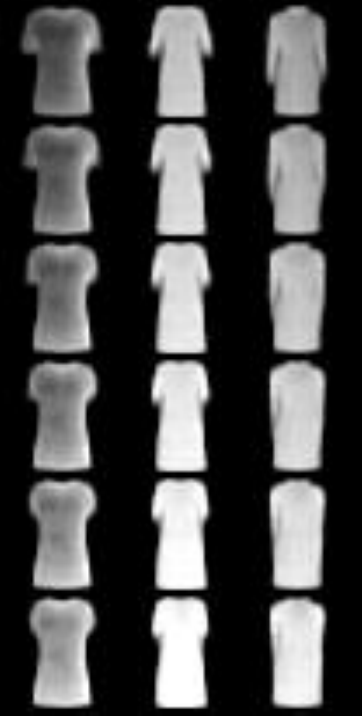}
      };
      \begin{scope}[x={(image.south east)},y={(image.north west)}]
        \draw[white] (0, 0.83) -- (1, 0.83);
      \end{scope}
    \end{tikzpicture}
    \\
    (a) & (b) & (c) & (d)\\
  \end{tabular}
  \caption{%
    Qualitative evaluation of sparsity.
    [Top] Average encoding magnitude over data for three example classes in \emph{Fashion-MNIST}.
    [Bottom] Latent interpolation (\(\downarrow\)) for different datapoints (top layer) along particular `active' dimensions.
    (a) Separation between the trouser legs (dim 49).
    (b) Top/Collar width of dresses (dim 30).
    (c) Shirt shape (loose/fitted, dim 19).
    (d) Style of sleeves across different classes---t-shirt, dress, and coat (dim 40).
  }
  \label{fig:sparsity-qualitative}
 \vspace*{-1ex}
\end{figure}

We explore the qualitative effects of sparsity in \cref{fig:sparsity-qualitative}, using a network trained with \(\alpha=1000, \beta=1,\) and \(\gamma=0.8\), corresponding to one of the models in \cref{fig:sparsity}~(left).
The top plot shows the average encoding magnitude for data corresponding to 3 of the 10 classes in the \emph{Fashion-MNIST} dataset.
It clearly shows that the different classes (trousers, dress, and shirt) predominantly encode information along different sets of dimensions, as expected for sparse representations (c.f.\ \cref{sec:experimentaldetails} for plots for all classes).
For each of these classes, we explore the latent space along a particular `active' dimension---one with high average encoding magnitude---to observe if they capture meaningful features in the image.
We first identify a suitable `active' dimension for a given instance (top row) from the dimension-wise magnitudes of its encoding, by choosing one, say~\(d\), where the magnitude far exceeds \(\sigma_0^2\).
Given encoding value~\(\z_d\), we then interpolate along this dimension (keeping all others fixed) in the range \((\z_d, \z_d + \mathrm{sign}(\z_d))\); the sign of~\(\z_d\) indicating the direction of interpolation.
Exploring the latent space in such a manner demonstrates a variety of consistent feature transformations in the image, both within class (a, b, c), and across classes (d), indicating that these sparse dimensions do capture meaningful features in the image.

Concurrent to our work,~\citet{tonolini2019variational} also considered imposing sparsity in \glspl{VAE} with a spike-slab prior (such that $\sigma_0\to0$).
In contrast to our work, they do not impose a constraint on the aggregate
encoder, nor do they evaluate their results with a quantitative sparsity metric that accounts for the varying length scales of different latent dimensions



\section{Discussion}
\label{sec:discussion}

\vspace*{-0.3ex}
\paragraph{Characterising Overlap}
\label{sec:overlap}

Precisely formalising what constitutes the level of overlap in the latent space is surprisingly subtle.
Prior work has typically instead considered controlling the level of compression through the mutual information between data and latents~\(I(\x;\z)\)~\citep{alemi2016deep,alemi2018fixing,Hoffman2016vz,phuong2018the}, with, for example,~\citep{phuong2018the} going on to discuss how controlling the compression can ``explicitly encourage useful representations.''
Although~\(I(\x;\z)\) provides a perfectly serviceable characterisation of overlap in a number of cases, the two are not universally equivalent and we argue that it is the latter which is important in achieving useful representations.
In particular, if the form of the encoding distribution is not fixed---as when employing normalising flows, for example---\(I(\x;\z)\) does not necessarily characterise overlap well.
We discuss this in greater detail in \cref{sec:app-overlap}.

However, when the encoder is unimodal with fixed form (in particularly the tail behaviour is fixed) and the prior is well-characterised by Euclidean distances, then these factors have a substantially reduced ability to vary for a given $I(\x;\z)$, which subsequently becomes a good characterisation of the level of overlap.
When $\q{\z|\x}$ is Gaussian, controlling the variance of $\q{\z|\x}$ (with a fixed $\q{\z}$) should similarly provide an effective means of achieving the desired overlap behaviour.
As this is the most common use case, we leave the development of more a general definition of overlap to future work, simply noting that this is an important consideration when using flexible encoder distributions.

\vspace*{-0.3ex}
\paragraph{Can VAEs Uncover True Generative Factors?}
\label{sec:unsup}

In concurrently published work,~\citet{locatello2018challenging} question the plausibility of learning unsupervised disentangled representations with meaningful features, based on theoretical analyses showing an equivalence class of generative models, many members of which could be entangled.
Though their analysis is sound, we posit a counterargument to their conclusions, based on the \emph{stochastic} nature of the encodings used during training.
Namely, that this stochasticity means that they need not give rise to the same~\gls{ELBO} scores (an important exception is the rotational invariance for isotropic Gaussian priors).
Essentially, the encoding noise forces nearby encodings to relate to similar datapoints, while standard choices for the likelihood distribution (e.g.\ assuming conditional independence) ensure that information is stored in the encodings, not just in the generative network.
These restrictions mean that the~\gls{ELBO} prefers smooth representations and, provided the prior is not rotationally invariant, means that there no longer need be a class of different representations with the same~\gls{ELBO}; simpler representations are preferred to more complex ones.

 The exact form of the encoding distribution is also important here.
 For example, imagine we restrict the encoder variance to be isotropic and then use a two dimensional prior where one latent dimension has a much larger variance than the other.
It will be possible to store more information in the prior dimension with higher variance (as we can spread points out more relative to the encoder variance).
 Consequently, that dimension is more likely to correspond to an important factor of the generative process than the other.
 Of course, this does not imply that this is a true factor of variation in the generative process, but neither is the meaning that can be attributed to each dimension completely arbitrary.

All the same, we agree that an important area for future work is to assess when, and to what extent, one might expect learned representations to mimic the true generative process, and, critically, when it should not.
For this reason, we actively avoid including any notion of a true generative process in our definition of decomposition, but note that, analogously to disentanglement, it permits such extension in scenarios where doing so can be shown to be appropriate.



\section{Conclusions}
\label{sec:conc}
\vspace*{-1ex}

In this work, we explored and analysed the fundamental characteristics of learning disentangled representations, and showed how these can be generalised to a more general framework of \emph{decomposition}~\citep{lipton2016mythos}.
We characterised the learning of decomposed latent representation with \glspl{VAE} in terms of the control of two factors:
\begin{inparaenum}[i)]
	\item overlap in the latent space between encodings of different datapoints, and
	\item regularisation of the aggregate encoding distribution to the given prior, which encodes the structure one would wish for the latent space to have.
\end{inparaenum}

Connecting prior work on disentanglement to this framework, we analysed the \acrshort{BVAE} objective to show that its contribution to disentangling is primarily through direct control of the level of overlap between encodings of the data, expressed by maximising the entropy of the encoding distribution.
In the commonly encountered case of assuming an isotropic Gaussian prior and an independent Gaussian posterior, we showed that control of overlap is the \emph{only} effect of the \acrshort{BVAE}.
Motivated by this observation, we developed an alternate objective for the \gls{ELBO} that allows control of the two factors of decomposability through an additional regularisation term.
We then conducted empirical evaluations using this objective, targeting alternate forms of decompositions such as clustering and sparsity, and observed the effect of varying the extent of regularisation to the prior on the quality of the resulting clustering and sparseness of the learnt embeddings.
The results indicate that we were successful in attaining those  decompositions.
\vspace*{-1ex}



\section*{Acknowledgements}
EM, TR, YWT were supported in part by the European Research Council under the European Union's Seventh Framework Programme (FP7/2007--2013) / ERC grant agreement no. 617071. TR research leading to these results also received funding from EPSRC under grant EP/P026753/1. EM was also supported by Microsoft Research through its PhD Scholarship Programme. NS was funded by EPSRC/MURI grant EP/N019474/1.

\bibliographystyle{plainnat}
\bibliography{references}
\clearpage
\newpage

\appendix

\section{Proofs for Disentangling the \acrshort{BVAE}}
\label{sec:bvae-thm}

\betaThe*
\vspace*{-3ex}
\begin{proof}
	Starting with~\cref{eq:beta-vae}, we have
	\begin{align*}
	\L_{\beta}(\x)
	=& \Ex[\q{\z \given \x}]{\log \p{\x \given \z}} + \beta H_{q_{\phi}} \\
	&+ \beta \Ex[\q{\z \given \x}]{\log \pz} \\
	=& \Ex[\q{\z \given \x}]{\log \p{\x \given \z}}
	+ (\beta - 1) H_{q_{\phi}} + H_{q_{\phi}} \\
	&+ \Ex[\q{\z \given \x}]{\log \pz^\beta - \log \Fz}
	+ \log \Fz \\
	=& \Ex[\q{\z \given \x}]{\log \p{\x \given \z}} + (\beta - 1) H_{q_{\phi}} \\
	&- \KL{\q{\z \given \x}}{\fz} + \log \Fz \\
	=& \L\left(\x ; \pi_{\theta,\beta},q_{\phi} \right)
	+ (\beta - 1) H_{q_{\phi}}
	+ \log \Fz
	\end{align*}
	as required.
\end{proof}

\gauss*
\vspace*{-2ex}
\begin{proof}
  We start by noting that
  \begin{align*}
    \gn{\x}
    &= \Ex[\fz]{\p{\x \given \z}}
     = \Ex[\pz]{\p{\x \given \z/\sqrt{\beta}}} \\
     &= \Ex[\pz]{p_{\theta'}\mleft(\x \given \z \mright)}
     = p_{\theta'}(\x)
  \end{align*}
  Now considering an alternate form of \(\mathcal{L}\left(\x ; \pi_{\theta,\beta},q_{\phi} \right)\) in~\cref{eq:bvae},
  \begin{align*}
    \mathcal{L}&\left(\x ; \pi_{\theta,\beta},q_{\phi} \right) \\
    =& \log \gn{\x} - \KL{\q{\z \given \x}}{\gn{\z \given \x}}
    \displaybreak[0] \\
    =& \log p_{\theta'}(\x)
      - \Ex[\q{\z \given \x}]{%
      \log\left(\frac{\q{\z \given \x}p_{\theta'}(\x)}{\p{\x \given \z} \fz}\right)
      } \displaybreak[0] \\
            \numberthis
            \begin{split}
    =& \log p_{\theta'}(\x) \\
      &-\Ex[q_{\phi'}(\z \given \x)]{%
      \log\left(\frac{\q{\z/\sqrt{\beta} \given \x}p_{\theta'}(\x)}{p_{\theta}(\x \given \z / \sqrt{\beta}) f_{\beta}(\z/\sqrt{\beta})}\right)
      }.
      \end{split}
      \label{eq:corr1-int1}
  \end{align*}
  We first simplify \(f_{\beta}(\z/\sqrt{\beta})\) as
  \begin{align*}
    f_{\beta}(\z/\sqrt{\beta})
    &= \frac{1}{\sqrt{2\pi \dmt{\Sigma/\beta}}}
      \exp\left(-\frac{1}{2}\z^T \Sigma^{-1} \z\right) \\
    &= p(\z) \beta^{(D/2)}.
  \end{align*}
  Further, denoting \(\z_\dagger = \z-\sqrt{\beta}\mu_{\phi'}(\x)\), and \(\z_\ddagger = \z_\dagger / \sqrt{\beta} = \z/\sqrt{\beta} - \mu_{\phi'}(\x)\), we have
  \begin{align*}
    q_{\phi'}(\z \given \x)
    =& \frac{1}{\sqrt{2\pi \dmt{S_{\phi}(\x)\beta}}}
      \exp\left(-\frac{1}{2\beta}\z_\dagger^T S_{\phi}(\x)^{-1} \z_\dagger\right), \\
    \q{\frac{\z}{\sqrt{\beta}} \given \x}
    &= \frac{1}{\sqrt{2\pi \dmt{S_{\phi}(\x)}}}
      \exp\left(-\frac{1}{2}\z_\ddagger^T S_{\phi}(\x)^{-1} \z_\ddagger\right) \\
    \text{giving}\quad&
    \q{\z/\sqrt{\beta} \given \x}
    = q_{\phi'}(\z \given \x) \beta^{(D/2)}.
  \end{align*}
  Plugging these back in to~\cref{eq:corr1-int1} while remembering $p_{\theta}(\x \given \z / \sqrt{\beta}) = p_{\theta'}(\x \given \z)$, we have
  \begin{align*}
    \mathcal{L}&\left(\x ; \pi_{\theta,\beta},q_{\phi} \right)\\
    &= \log p_{\theta'}(\x)
    - \Ex[q_{\phi'}(\z \given \x)]{%
    \log\left(%
    \frac{q_{\phi'}(\z \given \x)p_{\theta'}(\x)}{p_{\theta'}(\x \given \z) p(\z)}
    \right)
    } \\
    &=\L\mleft(\x ; \theta,\phi \mright),
  \end{align*}
  showing that the ELBOs for the two setups are the same.
  For the entropy term, we note that
  \begin{align*}
    H_{q_{\phi}}
    &= \frac{D}{2}\left(1 + \log 2\pi \right) + \frac{1}{2}\log \dmt{S_{\phi}(\x)} \\
    &= \frac{D}{2}\left( 1 + \log \frac{2\pi}{\beta}\right) +\frac{1}{2}\log \dmt{S_{\phi'}(\x)}.
  \end{align*}
  Finally substituting for $H_{q_{\phi}}$ and $\mathcal{L}\left(\x ; \pi_{\theta,\beta},q_{\phi} \right)$
  in~\cref{eq:bvae} gives the desired result.
\end{proof}

\begin{restatable}{corollary}{equivalent}
  Let $[\theta',\phi']=g_{\beta}([\theta,\phi])$ represent the transformation required to produced the rescaled networks in Corollary~\ref{cor:gauss}.
  If $0<\left|\det{\nabla_{\theta,\phi} g([\theta,\phi])}\right|<\infty \,\, \forall [\theta, \phi]$, then
  \begin{gather*}
    \nabla_{\theta,\phi} \mathcal{L}_{\beta}(\x ; \theta, \phi) = \mathbf{0} \,\,
    \Leftrightarrow \,\, \nabla_{\theta',\phi'}
    \mathcal{L}_{H,\beta}\left(\x ; \theta',\phi' \right)
    = \mathbf{0}.
  \end{gather*}
  Thus $[\theta^*\!,\phi^*]$ being a stationary point of $\frac{1}{n}\!\sum_{i=1}^{n}\! \mathcal{L}_{\beta}(\x_i;\theta,\phi)$ indicates that $g_{\beta}([\theta^*,\phi^*])$ is a stationary point of $\frac{1}{n} \sum_{i=1}^{n}\mathcal{L}_{H,\beta}\left(\x_i ; \theta',\phi' \right)$ and vice-versa.
\end{restatable}
%
\begin{proof}
  Starting from \cref{cor:gauss} we have that
  \begin{align*}
    \nabla_{\theta,\phi} \L_{\beta}(\x ; \theta, \phi)
    &\!=\!\nabla_{\theta,\phi} \L_{H,\beta}\left(\x ; \theta',\phi' \right) \\
    &\!=\! \left(\nabla_{\theta,\phi} g_{\beta}([\theta,\phi])\right) \nabla_{\theta',\phi'}
      \L_{H,\beta}\left(\x ; \theta',\phi' \right),
  \end{align*}
  so\!
  $\nabla_{\theta',\phi'} \L_{H,\beta}\!\left(\x ; \theta',\phi' \right) \!=\! \mathbf{0}
  \!\!\!\implies\!\!\! \nabla_{\theta,\phi} \L_{\beta}\!(\x ; \theta, \phi) \!=\! \mathbf{0}$ given our assumption that $\left|\det{\nabla_{\theta,\phi} g([\theta,\phi])}\right| \!<\! \infty \,\forall [\theta, \phi]$.
  Further, as $0<\left|\det{\nabla_{\theta,\phi} g([\theta,\phi])}\right| \,\forall [\theta, \phi]$, $\left(\nabla_{\theta,\phi} g_{\beta}([\theta,\phi])\right)^{-1}$ exists and has a finite determinant, so $\nabla_{\theta,\phi} \mathcal{L}_{\beta}(\x ; \theta, \phi)=\mathbf{0}$ also implies $\nabla_{\theta',\phi'} \mathcal{L}_{H,\beta}\left(\x ; \theta',\phi' \right)=\mathbf{0}$.
\end{proof}

\rotate*
\begin{proof}
	If $\z \sim q_{\phi}(\z | \x)$ and $\y=R \z$ then, by \citet[\S 8.1.4][]{petersen2008matrix}), we have
	\begin{align*}
	\y &\sim \mathcal{N}(\y ; R \mu_{\phi}(\x), R S_{\phi}(\x) R^T). 
	\end{align*}
	Consequently, the changes made by the transformed networks cancel to give the same reconstruction
	error as
	\begin{align*}
	\mathbb{E}_{q_{\phi}(\z|\x)}[\log p_{\theta}(\x \given \z)]
	&=\mathbb{E}_{q_{\phi^\dag}(\z|\x)}[\log p_{\theta}(\x \given R^T\z)] \\
	&= \mathbb{E}_{q_{\phi^\dag}(\z|\x)}[\log p_{\theta^\dag}(\x \given \z)].
	\end{align*}
	Furthermore, the KL divergence between $\q{\z|\x}$ and $\p{\z}$ is invariant to rotation,
	because of the rotational symmetry of the latter, such that \(\KL{q_{{\phi}}(\z|\x)}{p(\z)} = \KL{q_{{\phi}^\dag}(\z|\x)}{p(\z)}\).
	The result now follows from noting that the two terms of the
	\acrshort{BVAE} are equal under rotation.
\end{proof}

\section{Experimental Details}
\label{sec:experimentaldetails}

\begin{table}[ht!]
  \centering
  \begin{subtable}[b]{\columnwidth}
    \centering
    \scalebox{.8}{%
    \begin{tabular}{l@{\hspace*{3ex}}l}
      \toprule
      \textbf{Encoder}                & \textbf{Decoder}                   \\
      \midrule
      Input 64 x 64 binary image      & Input $\in \mathbb{R}^{10}$        \\
      4x4 conv. 32 stride 2 \& ReLU   & FC. 128 ReLU                       \\
      4x4 conv. 32 stride 2 \& ReLU   & FC. 4x4 x 64 ReLU                  \\
      4x4 conv. 64 stride 2 \& ReLU   & 4x4 upconv. 64 stride 2 \& ReLU    \\
      4x4 conv. 64 stride 2 \& ReLU   & 4x4 upconv. 64 stride 2 \& ReLU    \\
      FC. 128                         & 4x4 upconv. 32 stride 2 \& ReLU    \\
      FC. 2x10                        & 4x4 upconv. 1. stride 2            \\
      \bottomrule
    \end{tabular}}
    \caption{2D-shapes dataset.}
    \label{tab:2dshapes}
  \end{subtable}\\[3ex]
  \begin{subtable}[b]{\columnwidth}
    \centering
    \scalebox{0.8}{%
    \begin{tabular}{l@{\hspace*{3ex}}l}
      \toprule
      \textbf{Encoder}            & \textbf{Decoder}         \\
      \midrule
      Input \(\in \mathbb{R}^2\)  & Input $\in \mathbb{R}^2$ \\
      FC. 100. \& ReLU            & FC. 100 \& ReLU          \\
      FC. 2x2                     & FC. 2x2                  \\
      \bottomrule
    \end{tabular}}
    \caption{Pinwheel dataset.}
    \label{tab:pinwheel}
  \end{subtable}\\[3ex]
  \begin{subtable}[b]{\columnwidth}
  \centering
  \scalebox{0.8}{%
  \begin{tabular}{l}
    \toprule
    \textbf{Encoder}                                       \\
    \midrule
    Input 32 x 32 x 1 channel image                       \\
    4x4 conv.  32 stride 2 \& BatchNorm2d \& LeakyReLU(.2) \\
    4x4 conv.  64 stride 2 \& BatchNorm2d \& LeakyReLU(.2) \\
    4x4 conv. 128 stride 2 \& BatchNorm2d \& LeakyReLU(.2) \\
    4x4 conv. 50, 4x4 conv. 50                             \\
    \bottomrule
    \toprule
     \textbf{Decoder}                                      \\
    \midrule
    Input $\in \mathbb{R}^{50}$                            \\
    4x4 upconv. 128 stride 1 pad 0 \& BatchNorm2d \& ReLU        \\
    4x4 upconv.  64 stride 2 pad 1 \& BatchNorm2d \& ReLU        \\
    4x4 upconv.  32 stride 2 pad 1 \& BatchNorm2d \& ReLU        \\
    4x4 upconv. 1 stride 2 pad 1                                \\
    \bottomrule
  \end{tabular}}
  \caption{Fashion-MNIST dataset.}
  \label{tab:celeba}
  \end{subtable}
  \caption{Encoder and decoder architectures.}
  \vspace*{-3ex}
\end{table}

\paragraph{Disentanglement - 2d-shapes:}
\begin{figure*}[t]
  \centering
  \includegraphics[width=0.24\linewidth,trim={3em 0em 3em 0},clip]{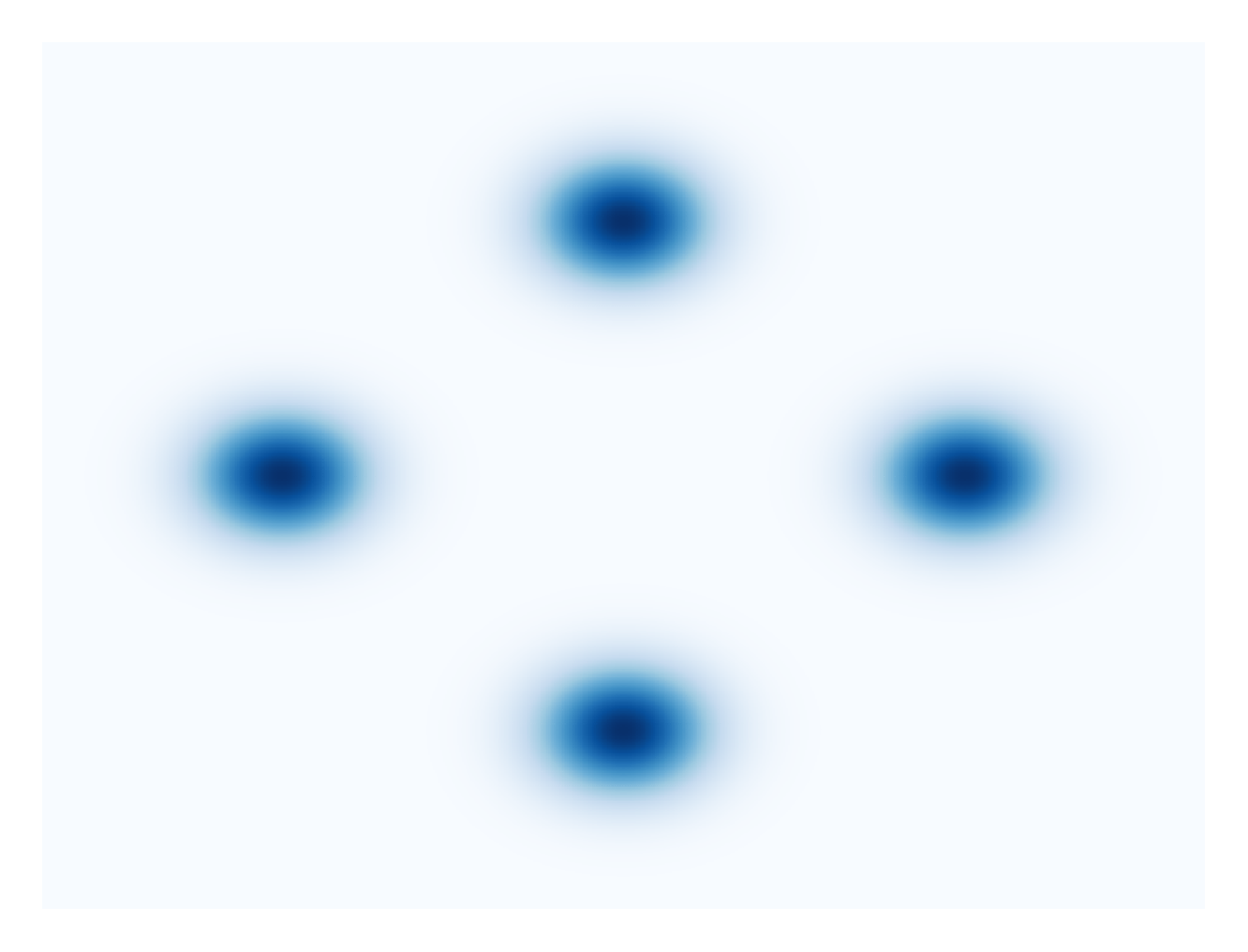}
  \quad
  \includegraphics[width=0.24\textwidth]{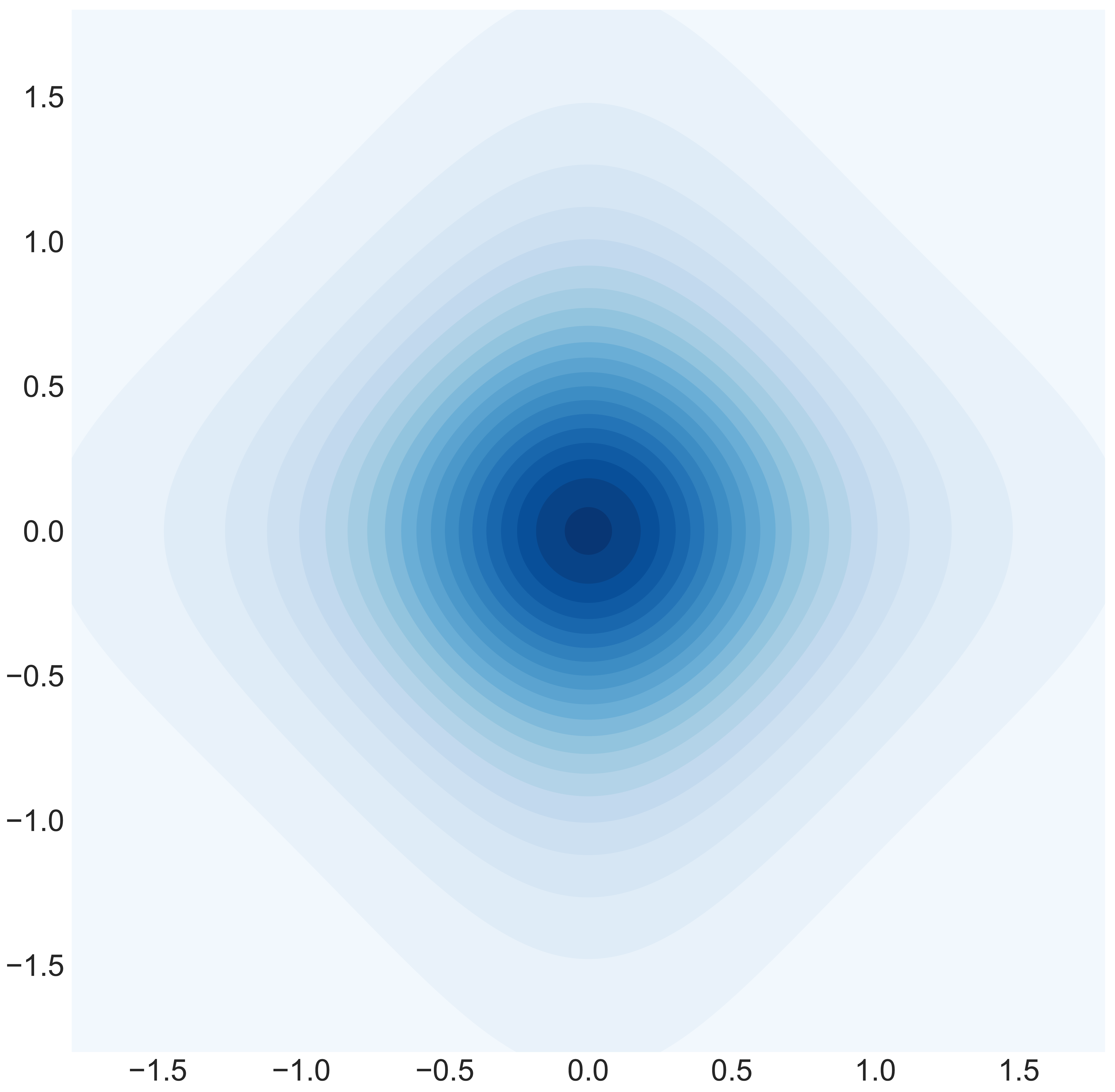}
  \includegraphics[width=0.24\textwidth]{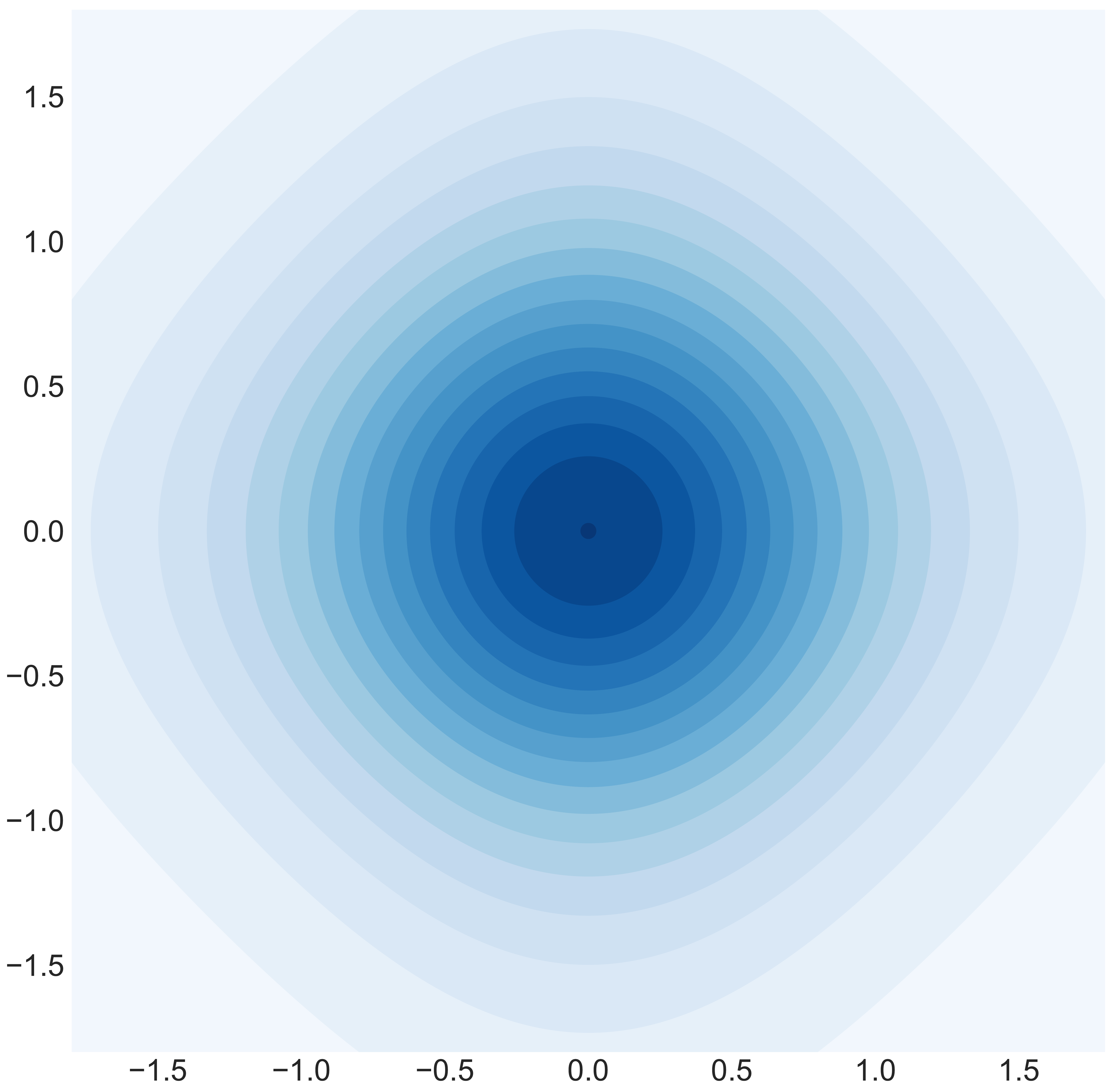}
  \includegraphics[width=0.24\textwidth]{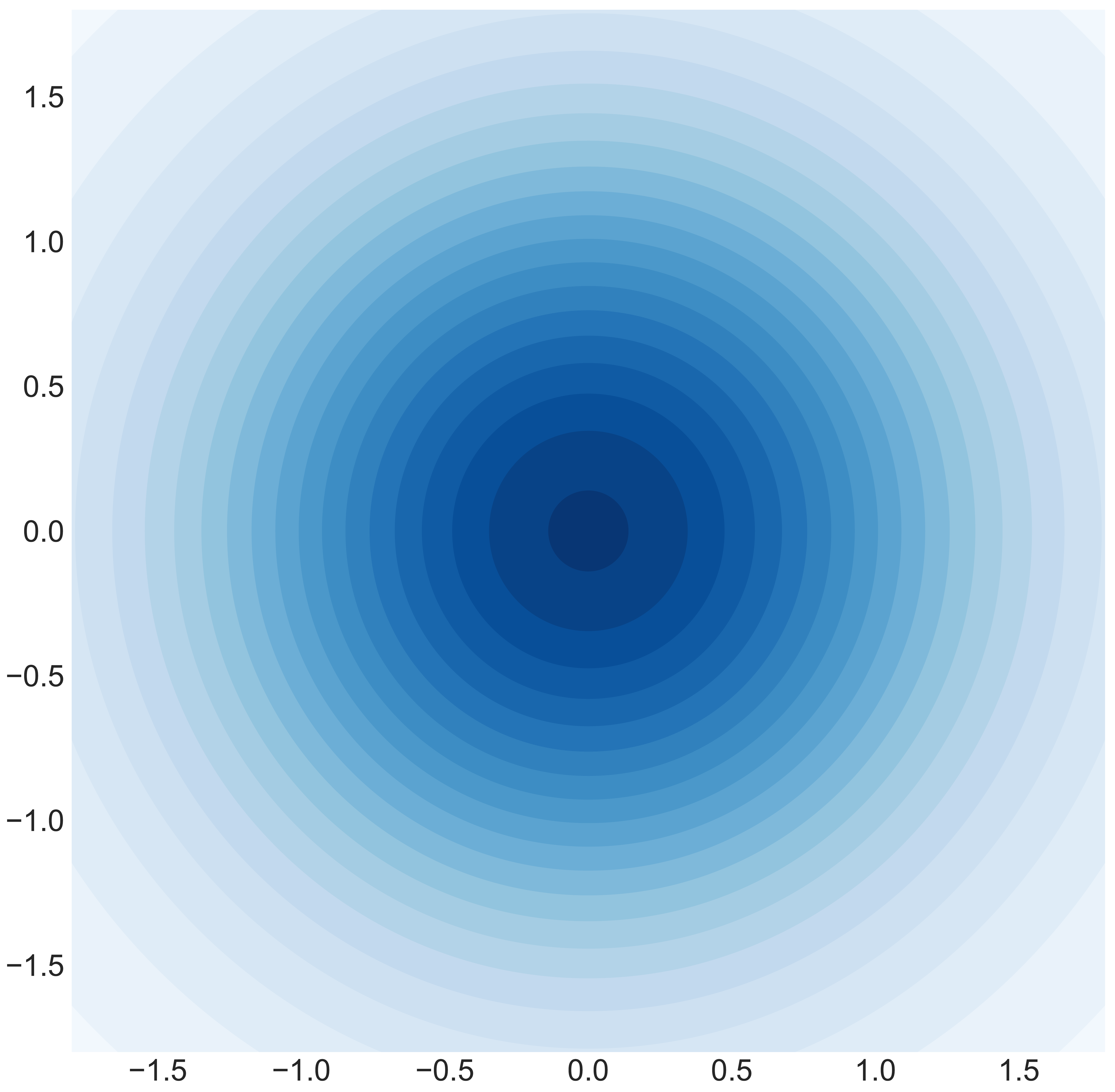}
  \\
  (a) MoG \hspace*{40ex} (b) Student-t \hspace{25ex}
  \caption{%
    (a) PDF of Gaussian mixture model prior $p(\z)$, as per \cref{eq:mog_prior}.
    (b) PDF for a 2-dimensional factored Student-t distributions $p_{\nu}$ with degree of freedom $\nu=\{3, 5, 100\}$ (left to right).
    Note that $p_{\nu}(\z) \to \mathcal{N}(\z;\mathbf{0},\mathbf{I})$ as $\nu\to\infty$.
  }
  \label{fig:priors}
  \vspace{-2ex}
\end{figure*}
The experiments from Section \ref{sec:experiments} on the impact of the prior in terms disentanglement are conducted on the \textbf{2D Shapes} \citep{dsprites17} dataset, comprising of
737,280 binary 64 x 64 images of 2D shapes with ground truth factors [number of values]: shape [3], scale [6], orientation [40], x-position [32], y-position [32].
We use a convolutional neural network for the encoder and a deconvolutional neural network for the decoder, whose architectures are described in \cref{tab:2dshapes}.
We use $[0,1]$ normalised data as targets for the mean of a Bernoulli distribution and negative cross-entropy for $\log p(\x|\z)$.
We rely on the Adam optimiser \citep{Kingma:2014us,Reddi:2018wc} with learning rate $1e^{-4}$, $\beta_1$ = 0.9, and $\beta_2$ = 0.999, to optimise the $\beta$-VAE objective from \cref{eq:bvae}.

For $\pz = \mathcal{N}(\z;\mathbf{0},\text{diag}(\mathbf{\sigma}))$, experiments were run with a batch size of $64$ and for $20$ epochs.
For $\pz = \prod_{d} \textsc{Student-t}(z_d;\nu)$, experiments were run with a batch size of $256$ and for $40$ epochs.
In \cref{fig:disentenglement}, the \emph{PCA initialised anisotropic} prior is initialised so that its standard deviations are set to be the first~$D$ singular values of the data.
These are then mapped through a softmax function to ensure that the $\beta$ regularisation coefficient is not implicitly scaled compared to the isotropic case.
For the \emph{learned anisotropic} priors, standard deviations are first initialised as just described, and then learned along with the model through a log-variance parametrisation.

We rely on the metric presented in \S4 and Appendix B of \citet{Hyunjik2018} as a measure of axis-alignment of the latent encodings with respect to the true (known) generative factors.
Confidence intervals in \cref{fig:disentenglement} were computed via the assumption of normally distributed samples with unknown mean and variance, with $100$ runs of each model.

\paragraph{Clustering - Pinwheel}
We generated spiral cluster data\footnote{\url{http://hips.seas.harvard.edu/content/synthetic-pinwheel-data-matlab}.}, with $n=400$ observations, clustered in 4 spirals, with radial and tangential standard deviations respectively of $0.1$ and $0.30$, and a rate of $0.25$.
We use fully-connected neural networks for both the encoder and decoder, whose architectures are described in \cref{tab:pinwheel}.
We minimise the objective from \cref{eq:alpha_obj}, with $\mathbb{D}$ chosen to be the inclusive \acrshort{KL} and $\q{\z}$ approximated by the aggregate encoding of the full dataset:
\begin{align*}
  \mathbb{D}&\left(\q{\z}, p(\z) \right)= \text{KL}\left(p(\z) || \q{\z} \right) \displaybreak[0] \\
  &= \Ex[{p(\z)}]{\log(p(\z)) - \log\left( \Ex[{p_\mathcal{D}(\x)}]{\q{\z \given \x}} \right)} \displaybreak[0] \\
  &\approx \sum_{j=1}^B  \left( \log p(\z_j) - \log \left( \sum_{i=1}^{n} \q{\z_j \given \x_i} \right) \right)
\end{align*}
with $\z_j \sim p(\z)$.
A Gaussian likelihood is used for the encoder.
We trained the model for 500 epochs using the Adam optimiser \citep{Kingma:2014us,Reddi:2018wc}, with $\beta_1=0.9$ and $\beta_2=0.999$ and a learning rate of $1e^{-3}$.
The batch size is set to $B=n$.

The Gaussian mixture prior (c.f.\ \cref{fig:priors}(a)) is defined as
\begin{align}
  p(\z)
  &= \sum_{c=1}^{C} \pi^c~\mathcal{N}(\z|\bm{\mu}^c, \bm{\Sigma}^c) \nonumber \\
  \label{eq:mog_prior}
  &= \sum_{c=1}^{C} \pi^c \prod_{d=1}^{D} \mathcal{N}(\z_d|\mu_d^c, \sigma_d^c)
\end{align}
with $D=2, \, C=4, \, \bm{\Sigma}^c=0.03I_{D}, \, \pi^c=1/C$, and $\mu_d^c \!\in\! \{0,1\!\}$.

\paragraph{Sparsity - Fashion-MNIST}
The experiments from Section \ref{sec:experiments} on the latent representation's sparsity
are conducted on the \textbf{Fashion-MNIST} \citep{xiao2017/online} dataset, comprising of $70,000$ greyscale images resized to 32\(\times\)32.

To enforce sparsity, we relied on a prior defined as a factored univariate mixture of a standard and low variance normal distributions:
\[
  p(\z) = \prod\nolimits_{d} ~(1 - \gamma) ~\mathcal{N}(z_d;0, 1) + \gamma ~\mathcal{N}(z_d;0, \sigma_0^2)
\]
with $\sigma_0^2=0.05$.
%
%
The weight, $\gamma$, of the low-variance component indicates how likely samples are to come from that component, hence to be \emph{off}.

We minimised the objective from \cref{eq:alpha_obj}, with $\mathbb{D}(\q{\z}, p(\z))$ taken to be a dimension-wise \gls{MMD} with a sum of \emph{Cauchy} kernels on each dimension.
Equivalently, we can think of this as calculating a single \gls{MMD} using the single kernel
\begin{align}
k(\x, \y) = \sum_{d=1}^D \sum_{\ell}^L
\frac{\sigma_{\ell}}{\sigma_{\ell=1} + (x_d - y_d)^2}.
\end{align}
where $\sigma_{\ell} \in \{0.2, 0.4, 1, 2, 4, 10\}$ are a set of length scales.

This dimension-wise kernel only enforces a congruence between the marginal distributions of $\x$ and $\y$ and so, strictly speaking, its \gls{MMD} does not
constitute a valid divergence metric in the sense that we can have
$\mathbb{D}(\q{\z}, p(\z))=0$ when $\q{\z}$ and $p(\z)$ are not identical
distributions: it only requires their marginals to match to get zero divergence. 

The reasons we chose this approach are twofold.
Firstly, we found that conventional kernels based on the Euclidean distance between encodings produced gradients with insurmountably high variances, meaning that effectively minimizing the divergence to get $\q{\z}$ and $\pz$ to match was not possible, even for very large batch sizes and $\alpha\to\infty$.

Secondly, though just matching the marginal distributions is not sufficient to ensure sparsity---as one could have some points with all dimensions close to the origin and some with all dimensions far away---a combination of the need to achieve good reconstructions and noise in the encoder process should prevent this from occurring.
In short, provided the noise from the encoder is properly regulated, there is little information that can be stored in latent dimensions near the origin because of the high level of overlap forced in this region.
Therefore, for a datapoint to be effectively encoded, it must have at least some of its latents dimensions outside of this region.
Coupled with the need for most of the latent values to be near the origin to match the marginal distributions, this, in turn, enforces a sparse representation.
Consequently, the loss in sparsity performance relative to using a hypothetical kernel that is both universal and has stable gradient estimates should only be relatively small, as is borne out in our empirical results.
This may, however, be why we see a slight drop in sparsity performance for very large values of $\alpha$.

We use a convolutional neural network for the encoder and a deconvolutional neural network for the decoder, whose architectures come from
the \acrshort{DCGAN} model \citep{DBLP:journals/corr/RadfordMC15} and are described in \cref{tab:celeba}.
We use $[0,1]$ normalised data as targets for the mean of a Laplace distribution with fixed scaling of $0.1$.
We rely on the Adam optimiser with learning rate $5e^{-4}$, $\beta_1$ = $0.5$, and $\beta_2$ = $0.999$. The model is then trained (on the training set) for $80$ epochs with a batch-size of $500$.

As an extrinsic measure of \emph{sparsity}, we use the \emph{Hoyer} metric \citep{Hurley2008ComparingMO}, defined for $\y \in \R^d$ by
\begin{align*}
\text{Hoyer}~(\y) = \frac{\sqrt{d} - \| \y \|_1 / \| \y \|_2}{\sqrt{d} - 1} \in [0, 1],
\end{align*}
yielding $0$ for a fully dense vector and $1$ for a fully sparse vector.
We additionally normalise each dimension to have a standard deviation of $1$ under its aggregate distribution, i.e. we use $\bar{z}_d = z_d / \sigma(z_d)$ where $\sigma(z_d)$ is the standard deviation of dimension $d$ of the latent encoding taken over the dataset.
Overall sparsity is computed by averaging over the dataset as \(\text{Sparsity} = 1/n \sum\nolimits_{i}^n \text{Hoyer}~(\bar{\z}_i)\).

As discussed in the main text, we use a trained model with \(\alpha=1000, \beta=1\), and \(\gamma=0.8\) to perform a qualitative analysis of sparsity using the \emph{Fashion-MNIST} dataset.
\Cref{fig:sparsity_all_labels} shows the per-class average embedding magnitude for this model, a subset of which was shown in the main text.
As can be seen clearly, the different classes utilise predominantly different subsets of dimensions to encode the image data, as one might expect for sparse representations.

\begin{figure*}[p]
  \centering
  \includegraphics[width=0.7\textwidth]{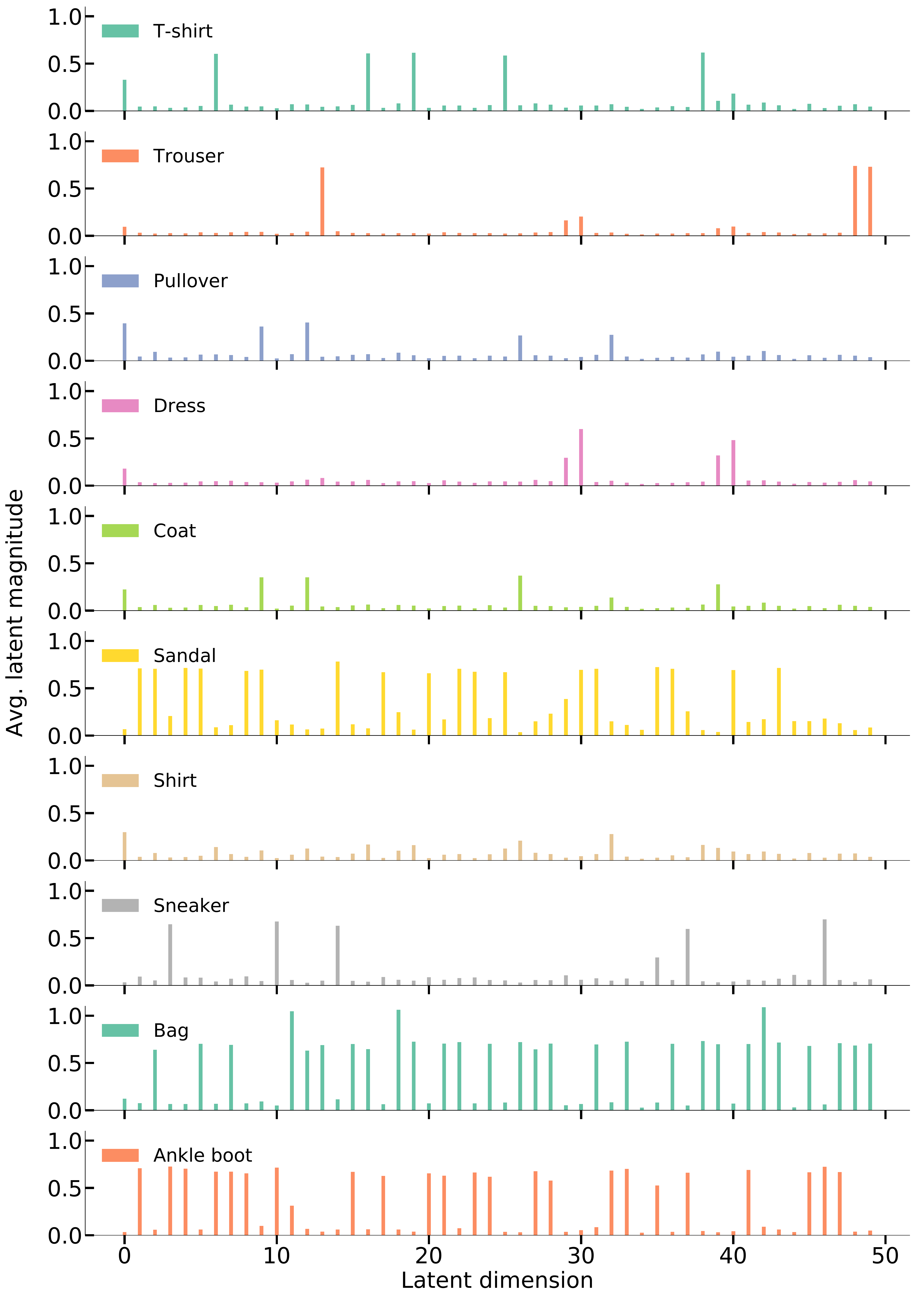}
  \vspace*{-1ex}
  \caption{%
  Average encoding magnitude over data for each classes in \emph{Fashion-MNIST}.
  }
  \vspace*{-4ex}
  \label{fig:sparsity_all_labels}
\end{figure*}

\section{Posterior regularisation}
\label{sec:posteriorreg}

The aggregate posterior regulariser $\mathbb{D}(q(\z), p(\z))$ is a little more subtle to analyse than the entropy regulariser as it involves both the choice of divergence and potential difficulties in estimating that divergence.
One possible choice is the exclusive \glsdesc{KL} divergence $\KL{q(\z)}{p(\z)}$, as previously used (without additional entropy regularisation) by~\citep{Esmaeili2018up,dilokthanakul2019explicit}, but also implicitly by~\citep{chen2018isolating}, through the use of a \gls{TC} term.
We now highlight a shortfall with this choice of divergence due to difficulties in its empirical estimation.

In short, the approaches used to estimate the $\text{H}[q(\z)]$ (noting that $\KL{q(\z)}{p(\z)} = -\text{H}[q(\z)] - \E_{q(\z)}[\log p(\z)]$, where the latter term can be estimated reliably by a simple Monte Carlo estimate) can exhibit very large biases unless very large batch sizes are used, resulting in quite different effects from what was intended.
In fact, our results suggest they will exhibit behaviour similar to the~\acrshort{BVAE} if the batch size is too small.
These biases arise from the effects of nesting estimators~\citep{rainforth2018nesting}, where the variance in the nested (inner) estimator for $q(\z)$ induces a bias in the overall estimator.
Specifically, for any random variable $\hat{Z}$,
\begin{align*}
  \E [\log (\hat{Z})]
  = \log(\E [\hat{Z}])-\frac{\text{Var} [\hat{Z}]}{2Z^2}+O(\varepsilon)
\end{align*}
where $O(\varepsilon)$ represents higher-order moments that get dominated asymptotically if $\hat{Z}$ is a Monte-Carlo estimator (see Proposition 1c in~\citet{maddison2017filtering},
Theorem 1 in~\citet{rainforth2018tighter}, or Theorem 3 in~\citet{domke2018importance}).
In this setting, $\hat{Z}=\hat{q}(\z)$ is the estimate used for $q(\z)$.
We thus see that if the variance of $\hat{q}(\z)$ is large, this will induce a significant bias in our \acrshort{KL} estimator.

To make things precise, we consider the estimator used for $\text{H}[q(\z)]$ in~\citet{Esmaeili2018up,chen2018isolating,dilokthanakul2019explicit}
\begin{subequations}
  \label{eq:naive}
  \begin{align}
    \hspace*{-1ex}%
    \text{H}[q(\z)]
    &\approx \hat{\text{H}}
      \triangleq -\frac{1}{B} \sum_{b=1}^{B}\log \hat{q}(\z_b), \text{\,where} \\
    \hat{q}(\z_b)
    &= \frac{\q{\z_b | \x_b}}{n}
      + \frac{n-1}{n(B-1)}\sum_{b'\neq b} \q{\z_b | \x_b'},
    \label{eq:hfvae-hatq}
  \end{align}
\end{subequations}
$\z_b \sim \q{\z | \x_b}$, and $\{\x_1,\dots,\x_B\}$ is the mini-batch of data used for the current iteration for dataset size $n$.
\citet{Esmaeili2018up} correctly show that $\E [\hat{q}(\z_b)] = \tilde{q}(\z_b)$, with the first term of~\cref{eq:hfvae-hatq} comprising an exact term in $\tilde{q}(\z_b)$ and the second term of~\cref{eq:hfvae-hatq} being an unbiased Monte-Carlo estimate for the remaining terms in $\tilde{q}(\z_b)$.

To examine the practical behaviour of this estimator when $B\ll n$, we first note that the second term
of~\cref{eq:hfvae-hatq} is, in practice, usually very small and dominated by the first term.
This is borne out empirically in our own experiments, and also noted in~\citet{Hyunjik2018}.
To see why this is the case, consider that given encodings of two independent data points, it is highly unlikely that the two encoding distributions will have any notable overlap (e.g. for a Gaussian encoder, the means will most likely be very many standard deviations apart), presuming a sensible latent space is being learned.
Consequently, even though this second term is unbiased and may have an expectation comparable or even larger than the first, it is heavily skewed---it is usually negligible, but occasionally large in the rare instances where there is substantial overlap between encodings.

Let the second term of \cref{eq:hfvae-hatq} be $T_2$ and the event that this it is significant be $E_S$, such that $\Ex{T_2 \given \neg E_s} \approx 0$.
As explained above, $\mathbb{P}(E_S) \ll 1$ typically.
We now have
\begin{align*}
  &\Ex{\hat{\text{H}}}\\
  &= \mathbb{P}(E_S) \Ex{\hat{\text{H}} \given E_S}
    + (1-\mathbb{P}(E_S)) \Ex{\hat{\text{H}} \given \neg E_S} \\
  &= \mathbb{P}(E_S) \Ex{\hat{\text{H}} \given E_S} +(1-\mathbb{P}(E_S)) \\
    &\quad\,\cdot\big(\!
    \log n
    \!-\!{\textstyle\frac{1}{B}\!\sum_{b=1}^{B}} \Ex{\log \q{\z_b | \x_b} | \neg E_S}
    \!-\! \Ex{T_2 | \neg E_S}
    \big) \\
  &= \mathbb{P}(E_S) \Ex{\hat{\text{H}} \given E_S}+(1-\mathbb{P}(E_S))\\
    &\quad\,\cdot\left(
    \log n
    \!-\! \Ex{\log \q{\z_1 | \x_1} | \neg E_S}
    \!-\! \Ex{T_2 | \neg E_S}
    \right) \\
  &\approx  \mathbb{P}(E_S) \Ex{\hat{\text{H}} \given E_S}\\
    &\quad+(1-\mathbb{P}(E_S))\! \left(\log n - \Ex{\log \q{\z_1 | \x_1}}\right)
\end{align*}
where the approximation relies firstly on our previous assumption that $\Ex{T_2 \given \neg E_S}\approx0$ and also that $\Ex{\log \q{\z_1|\x_1} \given \neg E_S}\approx \Ex{\log \q{\z_1|\x_1}}$.
This second assumption will also generally hold in practice, firstly because the occurrence of $E_S$ is dominated by whether two similar datapoints are drawn (rather than by the value of $\x_1$) and secondly because $\mathbb{P}(E_S)\ll1$ implies that
\begin{align*}
  &\Ex{\log \q{\z_1 \given \x_1}}\\
  &\quad= (1-\mathbb{P}(E_S))\Ex{\log \q{\z_1 \given \x_1} \given \neg E_S}\\
    &\quad\quad+\mathbb{P}(E_S)\Ex{\log \q{\z_1 \given \x_1} \given E_S}\\
  &\quad\approx\Ex{\log \q{\z_1 \given \x_1} \given \neg E_S}.
\end{align*}

Characterising $\Ex{\hat{\text{H}} \given E_S}$ precisely is a little more challenging, but
it can safely be assumed to be smaller than $\Ex{\log \q{\z_1 \given \x_1}}$, which is approximately what would result from all the $\x_b'$ being the same as $\x_b$.
We thus see that even when the event $E_S$ does occur, the resulting estimates will still, at most, be on a comparable scale to when it does not.
Consequently, whenever $E_S$ is rare, the $(1-\mathbb{P}(E_S)) \Ex{\hat{\text{H}} \given \neg E_S}$ term will dominate and we thus have
\begin{align*}
  \Ex{\hat{\text{H}}}
  &\approx \log n - \Ex{\log \q{\z_1 \given \x_1}}\\
  &= \log n + \Ex[p(\x)]{\text{H}[\q{\z \given \x}]}.
\end{align*}
We now see that the estimator mimics the $\beta-$VAE regularisation up to a constant factor $\log n$, as adding the $\E_{q(\z)}[\log p(\z)]$ back in gives
\begin{align*}
  -&\Ex{\hat{\text{H}}} - \E_{q(\z)}[\log p(\z)]\\
  &\approx \Ex[p(\x)]{\KL{\q{\z | \x}}{p(\z)}}-\log n .
\end{align*}
We should thus expect to empirically see training with this estimator as a regulariser to behave similarly to the $\beta-$VAE with the same regularisation term whenever $B\ll n$.
Note that the $\log n$ constant factor will not impact the gradients, but does mean that it is possible, even likely, that negative estimates for $\hat{\text{KL}}$ will be generated, even though we know the true value is positive.

Overcoming the problem can, at least to a certain degree, be overcome by using very large batch sizes $B$, at an inevitable computational and memory cost.
However, the problem is potentially exacerbated in higher dimensional latent spaces and larger datasets, for which one would typically expect the typical overlap of datapoints to decrease.

\subsection{Other Divergences}

As discussed in the main paper, $\KL{q(\z)}{p(\z)}$ is far from the only aggregate posterior regulariser one might use.
Though we do not analyse them formally, we expect many alternative divergence-estimator pairs to suffer from similar issues.
For example, using Monte Carlo estimators with the inclusive Kullback-Leibler divergence $\KL{p(\z)}{q(\z)}$ or the sliced Wasserstein distance~\citep{Kolouri:2018vo} both result in nested expectations analogously to $\KL{q(\z)}{p(\z)}$, and are therefore likely to similarly induce substantial bias without using large batch sizes.

Interestingly, however, \gls{MMD} and \gls{GAN} regularisers of the form discussed in~\citep{Tolstikhin:2017wy} do not result in nested expectations and therefore are necessarily not prone to the same issues: they produce unbiased estimates of their respective objectives.
Though we experienced practical issues in successfully implementing both of these---we found the signal-to-noise-ratio of the \gls{MMD} gradient estimates to be very low, particularly in high dimensions, while we experienced training instabilities for the \gls{GAN} regulariser---their apparent theoretical advantages may indicate that they are preferable approaches, particularly if these issues can be alleviated.
The \gls{GAN}-based approach to estimating the total correlation introduced by~\citet{Hyunjik2018} similarly allows a nested expectation to be avoided, at the cost of converting a conventional optimization into a minimax problem.

Given the failings of the available existing approaches, we believe that further investigation into divergence-estimator pairs for $\mathbb{D}(q(\z), p(\z))$ in \glspl{VAE} is an important topic for future work that extends well beyond the context of this paper, or even the general aim of achieving decomposition.
In particular, the need for congruence between the posterior (encoder), likelihood (decoder), and marginal likelihood (data distribution) for a generative model, means that ensuring $q(\z)$ is close to $p(\z)$ is a generally important endeavour for training \glspl{VAE}.
For example, mismatch between $q(\z)$ and $p(\z)$ will cause samples drawn from the learned generative model to mismatch the true data-generating distribution, regardless of the fidelity of our encoder and decoder.

\section{Characterising Overlap}
\label{sec:app-overlap}

Reiterating the argument from the main text, although the mutual information~\(I(\x;\z)\) between data and latents provides a perfectly serviceable characterisation of overlap in a number of cases, the two are not universally equivalent and we argue that it is overlap which is important in achieving useful representations.
In particular, if the form of the encoding distribution is not fixed---as when employing normalising flows, for example---\(I(\x;\z)\) does not necessarily characterise overlap well.

Consider, for example, an encoding distribution that is a mixture between the prior and a uniform distribution on a tiny \(\epsilon\)-ball around the mean encoding~\(\mu_{\phi}(x)\), i.e.
\(
\q{\z|\x}
\!=\! \lambda \cdot \text{Uniform}\left(\lVert\mu_{\phi}(\x)-\z\rVert_2 < \epsilon\right)
+ (1-\lambda) \cdot p(\z).
\)
If the encoder and decoder are sufficiently flexible to learn arbitrary representations, one now could arrive at \emph{any} value for mutual information simply by an appropriate choice of~\(\lambda\).
However, enforcing structuring of the latent space will be effectively impossible due to the lack of any pressure (other than a potentially small amount from internal regularization in the encoder network itself) for similar encodings to correspond to similar datapoints; the overlap between any two encodings is the same unless they are within $\epsilon$ of each other.

While this example is a bit contrived, it highlights a key feature of overlap that~\(I(\x;\z)\) fails to capture:~\(I(\x;\z)\) does not distinguish between large overlap with a small number of other datapoints and small overlap with a large number of other datapoints.
This distinction is important because we are particularly interested in \emph{how many} other datapoints one datapoint's encoding overlaps with when imposing structure---the example setup fails because each datapoint has the same level of overlap with all the other datapoints.

Another feature that~\(I(\x;\z)\) can fail to account for is a notion of \emph{locality} in the latent space.
Imagine a scenario where the encoding distributions are extremely multimodal with similar sized modes spread throughout the latent space, such as $q(\z|\x) = \sum_{i=1}^{1000} \mathcal{N}(\z ; \mu_{\phi}(\x)+m_i,\sigma I)$ for some constant scalar $\sigma$, and vectors $m_i$.
Again we can achieve almost any value for~\(I(\x;\z)\) by adjusting $\sigma$, but it is difficult to impose meaningful structure regardless as each datapoint can be encoded to many different regions of the latent space.



\end{document}